\documentclass[runningheads]{llncs}
\usepackage{algorithm}
\usepackage{enumitem}
\usepackage{newfloat}
\usepackage{listings}
\usepackage{tikz}
\usepackage{pgfplots}
\usepackage{pgfplotstable}
\usepackage{booktabs}
\usepackage{array}
\usepackage{graphicx} % Required for inserting images
\usepackage{amssymb}
\usepackage{orcidlink}
\renewcommand\orcidID[1]{\orcidlink{#1}}
\usepackage{tikz}
\usepackage{pgfplots}
\pgfplotsset{compat=1.18} 
\usepgfplotslibrary{groupplots}
\definecolor{lstm_rob}{HTML}{1F77B4}   % Strong Blue
\definecolor{lstm_nrob}{HTML}{AEC7E8}  % Light Blue
\definecolor{lstm_unk}{HTML}{708090}   % Slate Gray (Cool, ties to blue)
\definecolor{trans_rob}{HTML}{FF7F0E}  % Strong Orange
\definecolor{trans_nrob}{HTML}{FFBB78} % Light Orange
\definecolor{trans_unk}{HTML}{D3C5B8}  % Warm Gray (Ties to orange)

 \usetikzlibrary{automata, positioning, arrows.meta, shapes.geometric}
 
\usepackage{algpseudocode}
\usepackage{amsmath}
\title{Synthesis and Verification of Register Automata over Data Sequences}
\titlerunning{Synthesis and Verification of Register Automata}

\author{Chih-Duo Hong\inst{1}\orcidID{0000-0002-4064-8413} \and
Hongjian Jiang\inst{2}\orcidID{0009-0006-4082-2633} \and
Anthony W. Lin\inst{3,4}\orcidID{0000-0003-4715-5096} \and
Oliver Markgraf\inst{2}\orcidID{0000-0003-4817-4563} \and
Julian Parsert\inst{2}\orcidID{0000-0002-5113-0767} \and
Tony Tan\inst{5}\orcidID{0009-0005-8341-2004}}
\authorrunning{Hong et al.}
\institute{National Chengchi University, Taiwan \and
RPTU Kaiserslautern-Landau, Germany \and
Technical University of Munich, Germany \and
Max Planck Institute for Software Systems, Germany \and
University of Liverpool, United Kingdom}

\setlist{itemsep=1pt,topsep=3pt,parsep=0pt}

\usepackage{amsmath,mathtools} %one of them allows cases
\usepackage{bm} % allows for bold font vectors
\usepackage{algpseudocode}
\usepackage{graphicx}
\usepackage{multirow}

\usepackage{todonotes}

\usepackage{complexity}

\usepackage{amsmath}

\usepackage{tikz}
\usetikzlibrary{automata,positioning}

\ifdefined\spnewtheorem
\else
  \usepackage{amsthm}
  \newtheorem{theorem}{Theorem}

  \newtheorem{lemma}[theorem]{Lemma}
  \newtheorem{example}{Example}
  \newtheorem{claim}{Claim}
\fi

\newcommand{\OMIT}[1]{}

\newcommand{\pluseq}{\mathrel{+}=}

\newcommand{\setOfChains}[1]{\mathfrak{C}}

\newcommand{\Aut}{\mathcal{A}}

\newcommand{\Q}{\mathbb{Q}}

\newcommand{\controls}{Q}
\newcommand{\transrel}{\Delta}
\newcommand{\finals}{F}

\newcommand{\toolName}{\textit{RAExc}}

\newcommand{\nn}{\ensuremath{\mathcal{N}}}

\newcommand{\nnLSTM}{\ensuremath{\mathcal{N}_{\text{LSTM}}}}
\newcommand{\nnTransformer}{\ensuremath{\mathcal{N}_{\text{T}}}}

\newcommand{\learnLS}{\ensuremath{\mathcal{L}_{\text{LS}}}}

\newif\ifdraft\drafttrue
\ifdraft
\newcommand\anthony[1]{{\color{blue}
\tiny [#1 - \textbf{Anthony}]}}

\else
\newcommand\anthony[1]{}

\fi

\newcommand{\wt}{wt}

\usepackage{comment}

\newcommand{\edt}{\textup{edt}}

\newcommand{\curr}{\textit{curr}}
\newcommand{\acc}{\textit{acc}}

\newcommand{\mov}{\textsf{mov}}

\newcommand{\bft}{\mathbf{t}}
\newcommand{\cl}{\textup{cl}}
\newcommand{\bbR}{\mathbb{R}}

\newcommand{\cC}{\mathcal{C}}

\usepackage{pgfplots}
\pgfplotsset{compat=1.17}
\usepgfplotslibrary{statistics} % Add patterns library
\usepackage{pgfplotstable}

\usetikzlibrary{patterns}
\usetikzlibrary{matrix}

\usepackage{prettyref}
\newrefformat{sec}{Sec.~\ref{#1}}
\newrefformat{fig}{Fig.~\ref{#1}}
\newrefformat{lem}{Lemma~\ref{#1}}
\newrefformat{theo}{Theorem~\ref{#1}}
\newrefformat{alg}{Algorithm~\ref{#1}}
\newrefformat{app}{Appendix~\ref{#1}}
\newrefformat{tab}{Table~\ref{#1}}
\newrefformat{eg}{Example~\ref{#1}}

\newif\ifextended
\ifdefined\cameraready\extendedfalse\else\extendedtrue\fi
\newcommand{\appref}[2]{\ifextended\prettyref{#1}\else App.~#2 of~\cite{fullversion}\fi}

\begin{document}
\maketitle
\begin{abstract}
\setlength{\emergencystretch}{1em}
	Data sequences, such as time series, traces, and sensor streams, are words over large or infinite data domains. Their properties often depend on comparisons between values at different positions, which finite automata cannot express directly. We study synthesis and verification for Deterministic Register Automata (DRAs) over ordered data sequences. DRAs combine finite control with registers for remembered data values, providing compact symbolic descriptions of quantitative sequence languages. Our technical core is a local-search synthesis procedure and an exact robustness checker for DRAs. We fix the alphabet theory to ordered rationals $(\Q;<,=)$ with rational guard constants and use non-negative affine accumulator costs. For a fixed number of DRA registers, the checker runs in polynomial time; its dependence on that number is exponential. It reduces robustness to coverability and then to a shortest-path problem. Since the search ranges over $\Q^+$, it produces variable-length witnesses and, for RAA-definable alignment costs such as DTW, explores a counterexample space with no a priori length bound. As an application, we use learned DRAs as interpretable surrogates for queried neural sequence classifiers. A probably approximately correct (PAC) validation loop refines rejected candidates, returns genuine non-robustness witnesses for the classifier when found, and accepts a surrogate only with statistical guarantees for fidelity and local stability. We implement the framework in $\textsc{RAExc}$, enabling the synthesis of interpretable automata and the verification of their robustness. We evaluate the approach on benchmark languages and sequentialized MNIST, demonstrating high accuracy while verifying the robustness of the learned automata.
\end{abstract}

\section{Introduction}\label{sec:intro}

Data sequences are words whose positions carry values from an infinite, or very
large, data domain such as integers, rationals, or reals. They arise naturally in
execution traces, sensor streams, event logs, financial indicators, and other
time series data. Classical finite automata provide a canonical model of regular
languages over finite alphabets, but they cannot directly express properties
whose truth depends on comparing data values across positions, such as whether
successive peaks increase, whether future values exceed remembered thresholds,
or whether a signal preserves a quantitative trend. This limitation motivates
automata models over data words and data sequences, including data automata,
nominal automata, and register automata
\cite{DBLP:conf/lics/BojanczykMSSD06,two-variable-logic,learning-nominal-automata,BKL14,atom-book}.
Register automata extend finite control with finitely many registers that store
selected input values and compare later inputs with those stored values or with
constants. Over ordered domains such as $(\Q;<,=)$, they give a compact symbolic
language for global quantitative patterns while retaining useful algorithmic
properties, with applications in verification and synthesis
\cite{DBLP:conf/fsttcs/AbdullaAKR14,DBLP:conf/concur/KhalimovK19,FL22}.

This paper studies the synthesis and verification of deterministic register
automata (DRAs) over numeric data sequences. On the synthesis side, we ask how to
construct a small DRA from labelled sequence data or from a classifier that can
be queried on examples. On the verification side, we ask whether the language
recognised by a DRA is locally stable around a given data sequence under a
quantitative perturbation cost. These two questions support one another:
synthesis produces an interpretable symbolic model, and verification makes that
model an object on which robustness questions can be answered exactly.

A timely source of such queried classifiers is modern sequence modelling with
Recurrent Neural Networks (RNNs), LSTMs, and transformers
\cite{10.1162/neco.1997.9.8.1735,vaswani2017attention,time-series-survey}.
These models are widely used on sequence data, but their internal decision
procedures are difficult to inspect. A popular way to explain them is therefore
to extract an automaton that approximates the trained network and then analyse
that automaton instead of the original model. This idea has been studied for
RNNs, weighted automata, and transformers
\cite{WGY18,OWSH20,merrill2022extracting,zhang21,AM25}. Existing extraction
methods, however, predominantly target automata over finite alphabets. This is a
serious mismatch for numeric time series: discretising the values can obscure
comparisons between positions, while a finite automaton cannot remember and
compare data values in the way required by many natural quantitative rules.
DRAs are a natural replacement because their registers provide precisely the
finite memory needed for order and equality comparisons over data values.

For an automaton to be a useful explanation, agreement with the neural model is
not enough. The explanation should also be stable under perturbations that are
small according to the intended semantics of the data. If a surrogate changes its
classification after an imperceptible or semantically insignificant edit, then it
may describe a brittle boundary artefact rather than a robust pattern. We
therefore seek \emph{robust automata explanations}: compact DRAs that agree with
the queried model on the inputs of interest and whose classifications persist
through a neighbourhood of those inputs. The following example shows why numeric
comparisons call for DRAs and why robustness matters.

\begin{figure}[t]
	\centering
	\begin{minipage}{0.29\columnwidth}
		\centering
		\includegraphics[width=\linewidth]{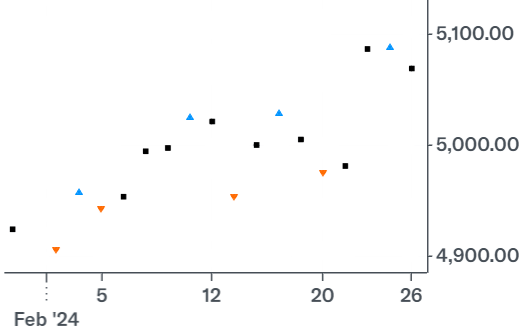}
	\end{minipage}
	\hspace{3.3em}
	\begin{minipage}{0.29\columnwidth}
		\centering
		\includegraphics[width=\linewidth]{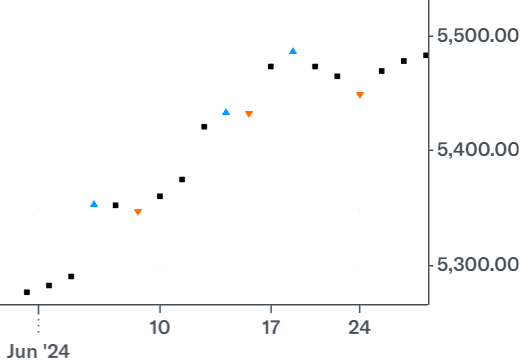}
	\end{minipage}
	\caption{Two S\&P\,500 price sequences. Red and blue triangles mark lows and highs, respectively. The left uptrend is robust to small last value decreases, while the right uptrend is fragile.}
	\label{fig:sp500}
\end{figure}

\begin{example}\label{eg:uptrend}
	Consider daily closing prices for a market index, represented by a sequence
	$x_1,\ldots,x_n$. We use the DRA in \prettyref{fig:motivating-example} as
	the formal definition of the ``higher highs and higher lows'' indicator. It
	first identifies a trough and a following peak; thereafter, it requires the
	price to remain above the most recently confirmed trough until a new, higher
	trough is identified, and it then requires a higher peak. Thus the DRA gives
	a streaming version of the usual higher-highs-higher-lows pattern, including
	a margin between confirmed extrema.
	\prettyref{fig:sp500} shows two historical \textsc{S\&P 500} price segments
	(Feb 1--26 and Jun 1--26, 2024). Both are classified as uptrends by this
	rule. However, under a small perturbation to the final price, the left
	sequence is \textit{robust}: the signal flips only if the price drops below
	the previous low. The right sequence is \textit{fragile}: even a slight drop
	can violate the DRA's running trough constraint and flip the signal. In practice, positive signals that are not robust
	warrant caution~\cite{murphy1999technical}.
\end{example}

The robustness question raised by \prettyref{eg:uptrend} exposes two limitations
of existing robustness checkers for neural sequence models~\cite{ko2019popqorn,du2021certRNN,jacoby2020verifying,mohammadinejad2020diffrnn,zhang2023rnnguard,lin2026rnnrobustness,shi2020transformer,liao2023transformers,hong2025transformer}.
First, most RNN checkers propagate bounds or abstractions through specific recurrent cells,
and transformer checkers reason about self-attention layers using convex relaxations,
linear bounds, or path constraints. These methods
can provide formal guarantees for the analysed neural architecture, but they
require access to the model structure, parameters, and
internal operations. Second, these checkers typically fix the shape of the input
tensor before verification. For sequence models, this means that the input length
is bounded in advance: the checker only perturbs the values of the original sequence.

This fixed-length view is appropriate for length-preserving $L_p$ perturbations
or for analyses with an explicit window bound, but it does not directly answer
the language-theoretic question whether there exists \emph{some} data sequence
$w \in \Q^+$ in the neighbourhood of the original sequence that flips the label.
The distinction matters for edit and alignment distances. Under edit costs,
counterexamples may have lengths different from the original input. Under
\emph{dynamic time warping} (DTW) and related alignment costs, the length of a
potential counterexample is not bounded by the original input length %unless an
%extra warping or length bound is imposed 
\cite{time-series-book}. Thus,
fixed-length neural checkers can miss the very counterexamples that sequence
level robustness is intended to rule out.

Our approach is to verify robustness on the extracted DRA itself. Perturbation
costs are specified by \emph{Register Automata with Accumulator} (RAA),
which scan a reference sequence and a candidate perturbed sequence while
accumulating their distance. Given a DRA $\Aut$, a sequence $v$, and a cost
function $D$ represented by an RAA, local robustness asks whether
$B_\delta(v) \cap L(\Aut^c)=\emptyset$, where
$B_\delta(v)=\{w\in\Q^+ : D(v,w)<\delta\}$. The search ranges over all data
sequences $w\in\Q^+$ rather than over a fixed-dimensional perturbation space.
Consequently, the checker can search for witnesses of variable length, and for
costs definable by RAAs such as DTW and edit distance,
it is not tied to any a priori bound on the
counterexample length. For a fixed number of registers, we reduce this robustness
problem to coverability for bounded projected RAAs and then to a single-source
shortest path computation. When robustness fails, the same construction yields a
concrete witness that flips the DRA label.

It remains to justify why the verified DRA explains the neural model. Our method
uses the neural model only as a labelling oracle. It is therefore
independent of the model architecture: the same extraction and validation loop
can be applied to LSTMs, transformers, or any other classifier over numeric data
sequences. The cost is that the guarantee transferred from the DRA to the neural
model is statistical rather than absolute. This is the standard model-agnostic
setting, where explanations are validated relative to a reference input
distribution. Popular post-hoc explainers such as LIME, Anchors, and SHAP
likewise characterise model behaviour through background samples or perturbation
distributions \cite{ribeiro2016why,ribeiro2018anchors,lundberg2017unified}. Our
loop learns a compact DRA from labelled samples, tests agreement with the neural
model on fresh samples, and invokes the exact DRA robustness checker on samples
where the two models agree. A failed robustness check either refines the learner
or, if the neural model also flips on the returned witness, produces a genuine
witness that the neural model is not robust. If the loop accepts, the DRA has
PAC statistical guarantees for both fidelity and local stability under the
reference distribution \cite{kearns1994introduction,Angluin1988Queries}.

\paragraph{Contributions} We make the following major contributions in this work.
\begin{itemize}[leftmargin=*]
	\item \textbf{Robustness verification for DRAs over data sequences.} We give a
	polynomial-time robustness checker for DRAs with a fixed number of registers
	over $(\Q;<,=)$, with perturbation costs specified by RAAs. Unlike fixed-length
	neural robustness checkers, the procedure searches over $w\in\Q^+$ and
	therefore supports variable-length counterexamples and, for metrics such as
	DTW, counterexample spaces of unbounded length.
	\item \textbf{Reduction to shortest paths.} We reduce DRA robustness to
	coverability for bounded projected RAAs and then to a single-source shortest-
	path problem. When robustness fails, the algorithm constructs a concrete
	label-flipping witness whose cost is within any prescribed $\epsilon>0$ of
	the minimal cost (the minimum itself need not be attained under strict
	guards).
	\item \textbf{DRA synthesis from black-box sequence classifiers.} We develop a
	local search learner that synthesises compact DRAs from labelled samples
	generated by a queried sequence classifier. The learner is
	architecture-agnostic.
	\item \textbf{PAC-validated robust explanations.} We combine learning and exact
	DRA robustness checking in a validation loop that refines rejected candidates,
	returns genuine non-robustness witnesses, or accepts a DRA with statistical
	guarantees for fidelity and local stability under a reference distribution.
	\item \textbf{Implementation and evaluation.} We implement the framework in
	\toolName{} and evaluate it on synthetic data languages and sequentialised
	MNIST, showing that compact DRAs can approximate neural sequence classifiers
	and expose non-robust behaviour in the learned surrogates.
\end{itemize}

\noindent\textbf{Organisation.} The rest of this paper is organised as follows.  \prettyref{sec:hhhl} illustrates how our automata learning and robustness verification procedures operate on a practically motivated example. \prettyref{sec:prelim} introduces register automata, their variants, and the
local robustness property. \prettyref{sec:robust} describes an efficient
robustness checking procedure for DRAs. \prettyref{sec:learning} presents our
learning algorithm, its integration with the robustness checker, and the
resulting statistical guarantees. \prettyref{sec:exp} reports our experimental
results. We discuss related work and conclude in \prettyref{sec:related-work}.

%!TEX root = main.tex

\section{Motivating Example}\label{sec:hhhl}
In this section, we elaborate on the higher-highs-higher-lows concept in \prettyref{eg:uptrend} and offer an intuitive account of how our robustness verification and automata learning algorithms operate on it. A DRA embodying this concept is depicted in \prettyref{fig:motivating-example}.
To simplify our presentation, we focus on the case where the price decreases first (since the case of increasing first is symmetric), and normalise the price data by setting the initial price to zero. %Also, we allow the price to stay at a peak and trough during an uptrend. %(e.g., when $x_1 < x_2 = x_3 > x_4$ holds, $x_2$ and $x_3$ together are regarded as a single peak).

%  \vspace{-0.4cm}
\begin{figure*}
	\centering
	\scalebox{0.77}{
		\begin{tikzpicture}[shorten >=1pt,node distance=4.1cm,on grid, auto,scale=0.6]
			\node[state,initial] (q_0)   {$q_0$}; 
			\node[state, accepting] (q_1) [right=of q_0] {$q_1$}; 
			\node[state, accepting] (q_2) [right=of q_1] {$q_2$}; 
			\node[state, accepting] (q_3) [right=of q_2] {$q_3$}; 
			\path[->] 
			(q_0) edge [loop, above, align=center] node {$r_1 \geq curr$\\
				$r_1 := curr$} (q_0)
			(q_1) edge  [loop, above, align=center] node {$r_2 \leq curr$\\
				$r_2 := curr$} (q_1)
			(q_2) edge  [loop, above, align=center] node {$r_3 \geq curr$, $r_1 < curr$\\
				$r_3 := curr$} (q_2)
			(q_3) edge  [loop, above, align=center] node {$r_3 \leq curr$\\
				$r_3 := curr$} (q_3)
			(q_0) edge  node {$r_1 < curr$} 
			node [below] {$r_2 := curr$} (q_1)
			(q_1) edge [above, align=center] node {$r_2 > curr$,\\$r_1 < curr$} 
			node [below] {$r_3 := curr$} (q_2)
			(q_2) edge  node {$r_3 < curr$} 
			node [below] {$r_1 := r_3, r_3 := curr$} (q_3)
			(q_3) edge [bend left, align=center] node {$r_3 > curr$, $r_2 < r_3$\\
				$r_2 := r_3$, $r_3 := curr$} (q_2);
		\end{tikzpicture}
	}
	\caption{A deterministic register automaton $\Aut$ that identifies the higher-highs-higher-lows price pattern.}
	\label{fig:motivating-example}
\end{figure*}

%This automaton recognizes stock prices characterized by ``higher highs and
%higher lows'', see Example~\ref{eg:uptrend}. 
%jOur learning algorithm can infer exactly this automaton
%jGiven enough samples, our learning
%jalgorithm can deduce precisely this automaton.

%This pattern is reminiscent of an ``uptrend'' observed in a stock price, where values continually rise over time. During an uptrend, each subsequent peak in the price is higher than the previous peak, and each subsequent trough in the price is higher than the previous trough. This pattern is therefore also known as ``higher peaks and higher troughs'' (see \cite{alma990009565880507476}).

%The automaton $\Aut$ has three registers $r_1,r_2,r_3$, which are initialized to zero. In principle, $\Aut$ uses $r_1$ to store the last trough, $r_2$ to store the last peak, and $r_3$ to store the previous value. The price keeps descending at state $q_0$ until a trough occurs, at which point $\Aut$ moves to state $q_1$. The price keeps ascending at $q_1$ until a peak occurs, at which point $\Aut$ moves to state $q_2$. Now, as long as the price stays above the last trough (stored in $r_1$), $\Aut$ stays at $q_2$ until it detects a higher trough and moves to state $q_3$. At state $q_3$, the price keeps ascending until it reaches a higher peak, at which point the automaton moves back to state $q_2$.

The DRA~$\Aut$ in \prettyref{fig:motivating-example} maintains three registers $r_1$, $r_2$, and $r_3$ (initialised to zero) representing the last trough, the last peak, and the previous value, respectively. It descends in $q_0$ until it reaches a trough, ascends in $q_1$ to a peak, and then alternates as follows: it remains in $q_2$ while the price stays above $r_1$; upon detecting a higher trough, it moves to $q_3$; and upon a higher peak, it returns to $q_2$.
As an illustration, consider the price sequence $v := 0, -1, 5, 3, 7, 9, 6, 8$, which exhibits peaks at $5$ and $9$ and troughs at $-1$, $3$, and $6$.
Since both the peaks and the troughs are strictly increasing, this sequence represents an uptrend and is therefore accepted by~$\Aut$.

%\tony{I used $a,b,c,d$ as values and $D$ as the distance function. I change all distance $d$ below to $D$.}

\paragraph{Robustness.}
Suppose we ask whether $v$ remains an uptrend when its final value is
perturbed by less than $\delta$: does some $w \in \Q^+ \setminus L(\Aut)$
satisfy $D(v,w) < \delta$, where
$D((v_1,\ldots,v_n),(w_1,\ldots,w_n)) := |v_n - w_n|$ is the
\emph{last-letter} distance? We build an automaton $\Aut_{D,v}$ with an
\emph{accumulator} register: on reading the $i$-th letter $w_i$ of
$w=w_1,\ldots,w_n$, the transition taken adds $a_iw_i+b_i$ to the
accumulator, where the rationals $a_i,b_i$ are fixed by that transition
(and may depend on $v$), so that
$\Aut_{D,v}(w) = \sum_{i=1}^n (a_i w_i + b_i)$. For the last-letter
distance, $a_i=b_i=0$ for $i<n$, while at $i=n$ the automaton takes
$(a_n,b_n)=(-1,v_n)$ if $v_n>w_n$ and $(a_n,b_n)=(1,-v_n)$ otherwise,
i.e.\ it adds $|v_n-w_n|$, giving $\Aut_{D,v}(w)=D(v,w)$.

Taking the product $\Aut'$ of $\Aut_{D,v}$ and $\Aut^c$ (the complement of
$\Aut$) gives $\Aut'(w) = D(v,w)$ and $L(\Aut') = L(\Aut^c) \cap \Q^n$,
reducing robustness to whether some $w \in L(\Aut')$ satisfies
$\Aut'(w) < \delta$ --- exactly what the procedure of
\prettyref{sec:robust} decides for general $\Aut$, $v$, and $\delta$. Here,
$v$ admits an adversarial instance iff $\delta > 8-3=5$, since $v$ stops
being an uptrend once the last value $8$ drops to $3$ or below.

\paragraph{Learning.}
Suppose we have an accurate trading model~$\nn$ that classifies the two sequences in \prettyref{fig:sp500} as uptrends.
Recall that the right-hand sequence is unstable under the last-letter distance: even an arbitrarily small perturbation ($\varepsilon\!>\!0$) of the final price violates the relevant running constraint and flips the signal.
For this sequence, our pipeline identifies a nearby sample as a witness of instability and refrains from certifying a robust DRA for~$\nn$.
In contrast, the left-hand sequence is stable under the same metric—its signal remains unchanged unless the last price falls below the previous low.
Consequently, we can learn a robust DRA for~$\nn$ over a neighbourhood of the left-hand sequence.
By explicitly targeting robustness, the learned automaton captures the intended pattern with a margin: agreed sequences are locally stable, spurious counterexamples are filtered, and genuine non-robust witnesses are reported.
This design turns the extracted DRA into an interpretable surrogate that provides both \emph{agreement} and \emph{stability} guarantees.

% !TEX root = ../main.tex

\section{Preliminaries}\label{sec:prelim}

We fix the ordered-rational alphabet theory $(\Q;<,=)$, with rational
constants in guards. $\Q^+$ denotes the non-empty rational
sequences; registers initially hold $0$.

\par\smallskip\noindent\textbf{Register automata.}
A \emph{$k$-register automaton} ($k$-RA) is
$\Aut=(\controls,\transrel,q_0,\finals)$, where $\controls$ is finite,
$q_0\in\controls$, $\finals\subseteq\controls$, and $\transrel$ is a finite
set of transitions $(p,\varphi,\psi,q)$. It has registers
$\bar r=(r_1,\ldots,r_k)$. A guard $\varphi(\bar r,\curr)$ is a conjunction
of atoms $x\sim y$, where $\sim\in\{<,\leq,=,\neq,\geq,>\}$ and
$x,y\in\bar r\cup\{\curr\}\cup\Q$. An assignment $\psi$ is a parallel
update in which every $r_i$ receives a constant, a register value, or the
current symbol.

A configuration is $(p,\nu)$ with $p\in\controls$ and
$\nu:\bar r\to\Q$. On reading $a$, the transition
$(p,\varphi,\psi,q)$ is enabled exactly when
$\varphi(\nu,a)$ holds; it changes $(p,\nu)$ to $(q,\psi(\nu,a))$.
A word is accepted if its run from $(q_0,\mathbf 0)$ ends in $\finals$; its
language is $L(\Aut)$. The RA is \emph{deterministic} (a DRA) if at most one
outgoing transition is enabled in every configuration. Missing transitions
reject. We totalise a DRA, when needed, by adding a non-accepting sink, and
then complement it by swapping final and non-final states. DRA languages are
closed under Boolean operations~\cite{atom-book}; write $\Aut^c$ for a DRA
recognising $\Q^+\setminus L(\Aut)$.

\par\smallskip\noindent\textbf{Register automata with accumulators.}
A \emph{$k$-register automaton with accumulator} ($k$-RAA)
$\Aut_D=(\controls,\transrel,q_0,\finals)$ reads a pair
$(v,w)\in\Q^+\times\Q^+$ with one head per sequence and computes a cost.
Its transitions $(p,\varphi,\psi,\chi,q,\mov)$ extend those of RAs: guards
and assignments may refer to the letters $\curr_1,\curr_2$ under the two
heads, $\mov\in\{1,2,\mathsf{both}\}$ names the head(s) that advance, and
$\chi:\acc\pluseq a_1\cdot\curr_1+a_2\cdot\curr_2+b$ ($a_1,a_2,b\in\Q$) is an affine
accumulator increment, required to be non-negative whenever $\varphi$ holds
(so that the shortest-path reduction of \prettyref{sec:robust} has
non-negative edges).

A configuration $(p,i,j,\nu)$ records the state, both head positions and the
register valuation; the run starts in $(q_0,1,1,\mathbf 0)$. A transition is
enabled if its guard holds and every head it reads or advances is not yet
exhausted; a transition that ignores an exhausted head may still fire, which
allows insertions and deletions at the end of a sequence. It updates the registers, advances the heads in
$\mov$ (the others re-read their letter), and adds $\chi(v_i,w_j)$ to the
cost. A run is accepting if it ends in a final state with both heads
exhausted. The accumulator is write-only: guards never test it. RAAs are
non-deterministic, which lets them choose, e.g., an alignment of $v$ and $w$;
$\Aut_D(v,w)$ is the minimum cost over all accepting runs ($\infty$ if there
is none), and $\Aut_D$ \emph{computes} $D$ if $\Aut_D(v,w)=D(v,w)$ for all
$(v,w)$. Registers store letters for later comparison, as in RAs; in
\prettyref{sec:robust}, the registers of the DRA under analysis become
registers of a product RAA. Define
\[
\wt(\Aut_D)=\inf_{v,w\in\Q^+}\Aut_D(v,w),\qquad
\wt(\Aut_D,v)=\inf_{w\in\Q^+}\Aut_D(v,w).
\]

\par\smallskip\noindent\textbf{Robustness.}
For $\delta>0$, let
$B_\delta(v)=\{w\in\Q^+:D(v,w)<\delta\}$. A DRA language $L(\Aut)$ is
$\delta$-robust at $v\in L(\Aut)$ iff
$B_\delta(v)\cap L(\Aut^c)=\emptyset$. Hence robustness asks whether a
word of cost below $\delta$ can reach the complement language. The RAAs computing
Hamming, DTW, edit, Manhattan, and last-letter distance are given in
\appref{app:metric-raas}{A}.

% !TEX root = main.tex
%
%
\section{Verifying Robustness of DRAs}\label{sec:robust}

In this section, we develop the decision procedure behind the following theorem.

%Our proof proceeds in two reductions: (i) we reduce robustness to a \emph{coverability} problem for a bounded projected RAA obtained via a product construction, and (ii) we further reduce this coverability instance to a single-source shortest-path computation on a finite graph. %We introduce the coverability notion and then present the constructions used in the reductions and their complexity implications.

%In this section, we show that the robustness of a $k$-DRA
%can be checked in polynomial time when $k$ is fixed.
%We state it formally as Theorem~\ref{theo:robustness}.

\begin{theorem}
\label{theo:robustness}
For any fixed number of registers $k$ and any fixed perturbation-cost RAA with
non-negative updates, checking whether a $k$-DRA is robust at a given sequence
is decidable in polynomial time. Moreover, whenever robustness fails, a
concrete witness to non-robustness can also be constructed in polynomial time.
\end{theorem}
The polynomial is in the automaton, the sequence, and the finite constant
set induced by the instance. The graph constructed below has
$O(|v||Q||C|^{k+1})$ vertices and $O(|v||\Delta||C|^{k+2})$ edges, so the
running time of our algorithm is exponential in $k$. Some super-polynomial
dependence on $k$ is expected: robustness subsumes DRA emptiness, which is PSPACE-complete when
$k$ is part of the input~\cite{DemriL09}, so no algorithm can be polynomial
jointly in the automaton size and $k$ unless $\mathrm P=\mathrm{PSPACE}$.

Our proof proceeds in two reductions: robustness of a $k$-DRA to
coverability for a restricted class of RAA, \emph{bounded projected RAA},
then coverability to single-source shortest paths. Both reductions run in polynomial
time, and since all accumulator updates (hence all graph edges) are
non-negative, Dijkstra's algorithm solves the shortest-path problem in
polynomial time, giving Theorem~\ref{theo:robustness}.

For the first reduction, we use \emph{coverability}, a thresholded \emph{min-cost} question in the sense of cost-register automata over finite alphabets~\cite{alur2013regular}; when accumulator updates are identically zero, it collapses to the usual reachability problem for register automata with linear arithmetic~\cite{chen2017register}. For an RAA $\Aut_D$, a sequence $v$, and a constant $\delta\in \Q$, we say that $\Aut_D$ is covered at $v$ by $\delta$ if $\wt(\Aut_D,v) < \delta$. The coverability problem asks, given $\Aut_D$, $v$, and $\delta$, whether $\Aut_D$ is covered at $v$ by $\delta$.
Robustness reduces to coverability: the product between $\Aut$ and $\Aut_D$
gives an RAA $\Aut_D'$ accepting $(v,w)$ iff $w\notin L(\Aut)$ and
$(v,w)\in L(\Aut_D)$, so $L(\Aut)$ is $\delta$-robust at $v$ iff $\Aut_D'$ is
not covered at $v$ by $\delta$. The reduction can be strengthened so that
$\Aut_D'$ is a \emph{bounded projected RAA}: for some $\alpha>0$, $\Aut_D'$
is \emph{$\alpha$-bounded} and \emph{$v$-projected} if $\Aut_D'(u,w)\neq
\infty$ implies $u=v$ and every letter $d$ occurring in $w$ satisfies
$|d|\leq \alpha$. The formal reduction is the following lemma.
% Note that the shortest witness of coverability (i.e.,~the shortest sequence
% $(v,w)$ satisfying $\AutP(v,w) < \mu$) can have length exponential in the number
% of registers, as illustrated in the following example.
%% \anthony{Isn't this $\Aut(w) < 0$?}
% To facilitate the first reduction, we need a few technical definitions.  We
% will actually show that the robustness problem can be reduced in
% polynomial time to the coverability problem for \emph{bounded projected PCA}
% defined as follows.
%
\begin{lemma}[Reducing robustness to coverability]
  \label{lem:robust-cover}
  Given a DRA $\Aut$, a sequence $v\in L(\Aut)$, an RAA $\Aut_D$, and a rational
  number $\delta > 0$, we can construct in polynomial time a rational number
  $\alpha>0$ and $\alpha$-bounded $v$-projected RAA $\Aut_D'$ such that
  $L(\Aut)$ is not $\delta$-robust at $v$ if and only if $\Aut_D'$ is covered at
  $v$ by $\delta$.
\end{lemma}
\begin{proof}
\noindent\emph{Product.} Restrict the first tape of $\Aut_D$ to $v$ by conjoining
$\curr_1=v_i$ to every transition that reads the $i$-th letter, tracking $i$
in the control state; every increment then has the form $a\cdot\curr_2+b$.
Totalise $\Aut$, let $\overline{\Aut}$ be its complement, and form the
product in which $\overline{\Aut}$ moves exactly when the second head
advances: a transition of $\Aut_D$ with $\mov=1$ leaves the
$\overline{\Aut}$-component unchanged, and one with
$\mov\in\{2,\mathsf{both}\}$ is paired with a transition of $\overline{\Aut}$
whose guard and update are evaluated at $\curr_2$. (A $\mov=1$ transition
rereads the current letter of $w$, which the DRA must read only once.) With
final states $F_{\overline{\Aut}}\times F_{\Aut_D}$, the product $\Aut_D''$
accepts $(v,w)$ with cost $\Aut_D(v,w)$ iff $\Aut_D$ accepts it and
$w\notin L(\Aut)$. It has polynomial size, non-negative increments, and the
registers of both automata. Hence $L(\Aut)$ is not $\delta$-robust at $v$ iff
$\Aut_D''$ is covered at $v$ by $\delta$.

\noindent\emph{Bounding.} Let $C_0$ contain $0$, the letters of $v$, and the
constants of $\Aut_D''$, and let
\[
\alpha \coloneqq 1+\max\bigl(\{|c|:c\in C_0\}\cup\{|b/a|,|(\delta-b)/a| :
a\curr_2+b \text{ an increment},\ a\neq0\}\bigr).
\]
On a run of cost below $\delta$, each increment satisfies
$0\leq a\curr_2+b<\delta$, which confines every letter read with $a\neq0$ to
$[-\alpha+1,\alpha-1]$. After propagating register assignments, every other
letter of $w$ occurs only in atoms $w_i\sim w_j$ or $w_i\sim c$ with
$c\in C_0$. Replacing these letters by rationals in $[-\alpha,\alpha]$ with
the same order type relative to the fixed letters and $C_0$, which exist by
density of $\Q$ (the margin $+1$ leaves room beyond all fixed values),
preserves every guard and the cost. Hence conjoining
$-\alpha\leq\curr_2\leq\alpha$ to every transition yields an
$\alpha$-bounded $v$-projected RAA $\Aut_D'$ that is covered at $v$ by
$\delta$ iff $\Aut_D''$ is.
\end{proof}

For the second reduction, we show that minimal-cost witnesses can be found
among sequences over a finite set of values. We need two auxiliary facts.

\begin{lemma}\label{lem:closure}
  Let $\Pi\subseteq\Q^n$, let $\cl(\Pi)$ be its closure in $\bbR^n$, let
  $\delta>0$, and let $f:\bbR^n\to\bbR$ be continuous. Then
  $f(x)<\delta$ for some $x\in\Pi$ iff $f(x)<\delta$ for some
  $x\in\cl(\Pi)$.
\end{lemma}
\begin{proof}
The forward direction is immediate. Conversely, if $x\in\cl(\Pi)$ and
$f(x)<\delta$, then by continuity $f<\delta$ on an open neighbourhood of
$x$, which contains a point of $\Pi$.
\end{proof}

\begin{lemma}\label{lem:register-provenance}
  Let $C\subseteq\Q$ contain $0$, the constants of an RAA $\Aut_D$, and the
  letters of $v$. On every run of $\Aut_D$ on $(v,w)$ with all letters of
  $w$ in $C$, every register value lies in $C$.
\end{lemma}
\begin{proof}
By induction along the run: registers start at $0\in C$, and each
assignment copies a constant, a register value, $\curr_1$ (a letter of $v$),
or $\curr_2$ (a letter of $w$), all of which lie in $C$.
\end{proof}

\begin{lemma}[Reducing coverability to shortest path]\label{lem:cover-path}
  %{\bf (Reduction from coverability to single source shortest path)}
  Let $k$ be a fixed positive integer.  Given $\alpha$-bounded $v$-projected
  $k$-RAA $\Aut_D$ and a rational number $\delta> 0$, we can construct in polynomial
  time a weighted graph $G=(V,E)$, a source vertex $s\in V$ and a set
  $U\subseteq V$ of target vertices such that $\Aut_D$ is covered at $v$ by
  $\delta$ iff there is a path in $G$ with weight $< \delta$ from the source vertex
  $s$ to some vertex in $U$.\footnote{A weighted graph $G=(V,E)$ is a directed
    graph where every edge has a weight. The weight of a path is the sum of the
    weights of its constituent edges.}
\end{lemma}
\begin{proof}
Since $\Aut_D$ is $v$-projected, we write each increment as $a \times \curr_2+b$.

\noindent\emph{Polytopes.} Fix a transition sequence $\bft=(t_1,\ldots,t_\ell)$
consuming $n$ letters of $w$; as $\mov=1$ transitions reread the current
letter, $t_i$ reads coordinate $w_{n_i}$, and several transitions may share
one. Let $\Pi(\bft)\subseteq\Q^n$ be the set of $w$ on which $\bft$ is an
accepting run on $(v,w)$. After substituting $v$ and propagating register
updates, every guard atom has the form $w_i\sim w_j$ or $w_i\sim c$, so
$\Pi(\bft)$ is a (not necessarily closed) polytope, and the cost is the
linear function $f_{\bft}(w)=\sum_{i}(a_iw_{n_i}+b_i)$. Hence $\Aut_D$ is
covered at $v$ by $\delta$ iff $f_{\bft}(w)<\delta$ for some $\bft$ and
$w\in\Pi(\bft)$, iff, by \prettyref{lem:closure}, the same holds for some
$w\in\cl(\Pi(\bft))$.

\noindent\emph{Finitely many values.} Let $C$ consist of $0$, the constants of
$\Aut_D$, the letters of $v$, and $\pm\alpha$. As $\Aut_D$ is
$\alpha$-bounded, $\cl(\Pi(\bft))$ is bounded, so $f_{\bft}$ attains its
minimum over it at a vertex $w$. Connect two coordinates of $w$ whenever a
constraint that is tight at $w$ relates them. Every connected component
contains a coordinate tied by a tight constraint to a constant in $C$:
otherwise, shifting the whole component by $\pm\epsilon$ preserves all tight
constraints and, for small $\epsilon$, all others, so $w$ would be the
midpoint of two feasible points. Hence every coordinate of $w$ lies in $C$,
and by \prettyref{lem:register-provenance} so does every register value
along the run.

\noindent\emph{The graph.} $G$ enumerates runs over values in $C$. A vertex
$(p,i,\mu,z,\tau)$ records the control state $p$, the first-head position
$i$, the register valuation $\mu\in C^k$, the value $z\in C\cup\{\star\}$
given to the current letter of $w$ ($\star$ if it has not been read yet, so
that rereads by $\mov=1$ transitions agree), and a tag $\tau$ that orders
the infinitesimal offsets of values sharing a base value in $C$; the tag
decides strict and disequality atoms at boundary values, and the number of
tags depends only on $k$. Every transition of $\Aut_D$, together with a
value in $C$ for the current letter when $z=\star$, yields an edge if its
guard holds under the tag, weighted by its non-negative increment. The
source $s$ is the initial configuration, and $U$ contains the vertices with
a final state, $i=|v|+1$, and $z=\star$, i.e.\ both heads exhausted. The tag
semantics is spelled out in \appref{app:cover-path-graph}{B}.

\noindent\emph{Correctness and size.} ($\Rightarrow$) Given an accepting run of cost
below $\delta$ on $\bft$, the minimising vertex $w'$ of $\cl(\Pi(\bft))$ has
coordinates in $C$ and $f_{\bft}(w')<\delta$; the direction from $w'$ into
$\Pi(\bft)$ yields tags satisfying all strict atoms, so following $\bft$
gives a path from $s$ to $U$ of weight $f_{\bft}(w')<\delta$. ($\Leftarrow$)
A path of weight below $\delta$ determines $\bft$ and a point
$w'\in\cl(\Pi(\bft))$ with $f_{\bft}(w')<\delta$; its tags give a rational
perturbation into $\Pi(\bft)$, which by \prettyref{lem:closure} keeps the
cost below $\delta$. Finally, $G$ has $O(|v||Q||C|^{k+1})$ vertices and
$O(|v||\Delta||C|^{k+2})$ edges, which is polynomial for fixed $k$.
\end{proof}
Backtracking a shortest path yields, for every $\epsilon>0$, a concrete
witness whose cost is within $\epsilon$ of the minimal cost. The minimum
itself need not be attained when strict guards are tight, cf. \appref{app:cover-path-graph}{B}.

\begin{proof}[Proof of Theorem~\ref{theo:robustness}]
By \prettyref{lem:robust-cover}, non-robustness of $L(\Aut)$ at $v$ reduces
in polynomial time to coverability of a bounded projected RAA with $k$ plus
the cost-RAA's registers, a fixed number. By \prettyref{lem:cover-path}, this
reduces in polynomial time to a shortest-path query on a graph with
non-negative edge weights, solved by Dijkstra's algorithm; backtracking the path yields
the witness.
\end{proof}

For cost functions such as Hamming distance, the witness $w$ is forced to
have the same length as $v$, so coverability can be solved by unrolling to
that fixed length --- an acyclic computation, polynomial in $k$ as well.
This fails once the constructed RAA permits witnesses of unbounded length
(i.e.\ the unrolled computation may contain cycles), hence the graph construction above.

%In summary, this section shows that DRA robustness can be checked effectively: robustness reduces to coverability for a bounded projected RAA and then to a shortest-path computation, yielding a polynomial-time procedure for a fixed number of registers. These results provide the symbolic subroutine used later in our extraction loop to certify or refute robustness of DRAs.

\section{Synthesis of Robust DRAs}\label{sec:learning}\label{sec:extracting}
This section presents the full extraction pipeline: synthesising candidate
DRAs and validating them as robust surrogates. Candidates are learned by a
simulated-annealing heuristic, \learnLS, accepting mutations with a
temperature-dependent probability so score-decreasing moves stay available
and the search escapes local optima. At each round, a certification interface
returns a genuine non-robustness witness, refines the labelled
sample set, or accepts the candidate with statistical
guarantees~\cite{kearns1994introduction}.

\subsection{Synthesis Based on Simulated Annealing}
For fixed numbers of states, registers, and constants, the candidate DRA space
is finite. We search this space by simulated annealing: a candidate is scored by
its classification accuracy on the labelled sample set $S=(S^{+},S^{-})$,
optionally up-weighting a distinguished ``golden'' positive subset
$G\subseteq S^{+}$,
\begin{footnotesize}
\[
\text{score}(\Aut, S)
= \frac{|S^{+}\cap L(\Aut)| + |S^{-}\setminus L(\Aut)|
	+ w\,|G\cap L(\Aut)|}
{|S^{+}| + |S^{-}| + w\,|G|},
\;\;
w = \begin{cases}
	\frac{|S^{+}| + |S^{-}|}{10\,|G|}, & G\neq\varnothing,\\[.5ex]
	0, & G=\varnothing.
\end{cases}
\]
\end{footnotesize}
The golden subset thus contributes a fixed tenth of the weight of $S$,
irrespective of its size, and its elements are counted twice, once through
$S^{+}$ and once through $G$; when $G=\varnothing$ the score reduces to plain
classification accuracy.

The search space is finite because guards and assignments are drawn from
finite menus. Fix a finite \emph{constant pool} $K\subseteq\Q$. A
\emph{guard atom} is an interval constraint $\ell\bowtie_1\curr\bowtie_2u$
with $\bowtie_1,\bowtie_2\in\{<,\leq\}$ and bounds
$\ell,u\in\bar r\cup K\cup\{\pm\infty\}$, and an assignment maps each
register to an element of $\bar r\cup\{\curr\}\cup K$; the \emph{atom space}
is the finite set of such guard--assignment pairs. At each step, the learner
selects a weighted mutation: resampling a state's outgoing transitions from
the atom space, flipping/optimising accepting states, or splitting a frontier
transition to activate a fresh state; non-accepted moves are reverted in
place, and a near-neutral split (activating a state while lowering the score
by at most one sample's weight) is retained. \prettyref{alg:sa} gives the
full procedure, keeping the best automaton $\Aut^*$ separate from the
current $\Aut$. \learnLS{} is a heuristic: it returns the
best-scoring DRA found within its budget, which need not be optimal and may be
a local optimum. All guarantees we state for extracted DRAs come from the
validation phase below, not from the search.

\begin{algorithm}[htb]
	\footnotesize
	\caption{\textsc{SimulatedAnnealing}$(S,\Aut_0,T_0,\rho,T_{\min},\textit{max\_iter},\textit{max\_time})$}
	\label{alg:sa}
	\begin{algorithmic}[1]
		\State $\Aut \leftarrow \Aut_0$; $\Aut^* \leftarrow \Aut_0$; $T \leftarrow T_0$; $i \leftarrow 0$; $t_{\mathrm{start}} \leftarrow \mathrm{now}()$
		\While{$i < \textit{max\_iter}$ \textbf{and} $\mathrm{elapsed}(t_{\mathrm{start}}) < \textit{max\_time}$}
			\State $\Aut' \leftarrow \textsc{Mutate}(\Aut)$ \Comment{weighted choice of the three operators above}
			\State $\Delta \leftarrow \text{score}(\Aut',S) - \text{score}(\Aut,S)$
			\If{$\Delta \geq 0$ \textbf{or} Bernoulli$(\exp(\Delta/T))$ succeeds}
				\State $\Aut \leftarrow \Aut'$ \Comment{annealing acceptance}
			\ElsIf{$\Aut'$ activates a fresh state \textbf{and} $\Delta \geq -1/|S|$}
				\State $\Aut \leftarrow \Aut'$ \Comment{near-neutral split, retained}
			\EndIf
			\State \textbf{if} $\text{score}(\Aut,S) > \text{score}(\Aut^*,S)$ \textbf{then} $\Aut^* \leftarrow \Aut$
			\State $T \leftarrow \rho \cdot T$; \textbf{if} $T < T_{\min}$ \textbf{then} $T \leftarrow T_0$ \Comment{reheat}
			\State $i \leftarrow i+1$
		\EndWhile
		\State \Return $\Aut^*$
	\end{algorithmic}
\end{algorithm}

\subsection{PAC Validation and End-to-End Extraction}\label{sec:validation-extraction}

We couple simulated-annealing learning with black-box validation. The user
fixes a target agreement threshold $p\in(0,1)$, an accuracy tolerance
$\gamma\in(0,1)$ (the slack allowed between $p$ and the guaranteed
agreement), and an error probability $\varepsilon\in(0,1)$. For each
candidate DRA, draw $n=\lceil\ln(1/\varepsilon)/(2\gamma^2)\rceil$ fresh samples
from the reference distribution $\mathcal D$ and permit at most
$d_{\max}=\lfloor n(1-p)\rfloor$ disagreements with $\nn$. On each agreeing
sample, invoke the checker from \prettyref{sec:robust}. A label-flipping
neighbour on which $\nn$ also flips is returned as a genuine witness. The
candidate is accepted iff it has at most $d_{\max}$ disagreements over the $n$
samples and the checker finds no label-flipping neighbour whose label under
$\nn$ equals that of its centre; otherwise every such neighbour, together with
the disagreement samples, refines the labelled set before the next synthesis
iteration. An accepted candidate is a
statistically faithful and locally stable surrogate in the sense of
\prettyref{thm:end-to-end}.

\paragraph{Equivalence Checking}

Exact equivalence against a black-box model with no equivalence oracle is
generally unavailable, so our equivalence phase adopts a
probably-approximately-correct (PAC) view~\cite{kearns1994introduction}:
replace the oracle with i.i.d.\ draws from the reference distribution
$\mathcal D$, accept only if the empirical disagreement stays below a
threshold, and lift this via a one-sided Hoeffding bound to a
distribution-level fidelity guarantee. Accordingly, \textsc{Accept} means PAC-validated fidelity under $\mathcal D$, not exact equivalence between $\Aut$ and $\nn$.

\begin{lemma}[PAC fidelity guarantee]\label{thm:approx-eq}
	Let $p,\varepsilon,\gamma\in(0,1)$ and consider one validation phase for a
	candidate DRA $\Aut$, with
	$n=\lceil \ln(1/\varepsilon)/(2\gamma^2)\rceil$ and
	$d_{\max}=\lfloor n(1-p)\rfloor$. If the phase accepts $\Aut$, then, with
	probability at least $1-\varepsilon$ over its random sampling process,
	$
	\Pr_{w\sim\mathcal D}[\nn(w)=\Aut(w)] \ge p-\gamma.
	$
\end{lemma}

\begin{proof}
	Condition on the candidate DRA selected before this validation phase. Let
	$X_i=\mathbf 1[\nn(w_i)\neq\Aut(w_i)]$ and
	$q=\Pr_{w\sim\mathcal D}[\nn(w)\neq\Aut(w)]$. The $X_i$ are i.i.d.
	Bernoulli variables with mean $q$. If the phase accepts, then
	$\frac1n\sum_{i=1}^nX_i\leq 1-p$. Hence, if $q>1-p+\gamma$, acceptance
	implies $\frac1n\sum_{i=1}^nX_i-q< -\gamma$. Hoeffding's inequality therefore
	gives $
	\Pr\!\left[\text{the phase accepts}\right]
	\leq e^{-2n\gamma^2}\leq\varepsilon.
	$
	Thus the probability that the phase accepts while
	$q>1-p+\gamma$ is at most $\varepsilon$. Equivalently, with probability at
	least $1-\varepsilon$, acceptance implies the stated agreement bound.
\end{proof}

\paragraph{Robustness Checking}
The validation procedure checks local stability only on samples where
$\Aut$ and $\nn$ agree. For such a sample $w$, the checker from
\prettyref{sec:robust} searches $B_\delta(w)$ for a point on which
$\Aut$ flips. If the network flips as well, the algorithm returns a
genuine witness; otherwise the point is a refinement example for the learner.
We write
$ \mathrm{Stab}_\delta(\Aut,w) \coloneqq \forall w'\in B_\delta(w): \Aut(w')=\Aut(w) $.

\begin{lemma}[PAC instability-mass bound on agreed inputs]\label{thm:robust-matching-positives}
	Assume the checker of \prettyref{sec:robust} is complete for $B_\delta$, that
	is, on every queried input $w$ it reports some $w'\in B_\delta(w)$ with
	$\Aut(w')\neq\Aut(w)$ whenever such a $w'$ exists. Let
	$
	\lambda \coloneqq \Pr_{w\sim \mathcal D}[\nn(w)=\Aut(w) \wedge \neg
	\mathrm{Stab}_\delta(\Aut,w)].
	$
	Suppose a validation phase checks $(\nn,\Aut)$ with sample size $n$. For any
	confidence parameter $\eta\in(0,1)$, if the phase accepts then, with
	confidence at least $1-\eta$, it holds that $\lambda \le 1-\eta^{1/n}$.
\end{lemma}

\begin{proof}
	Condition on the candidate DRA selected before this validation phase. For each
	sample, let $Z_i$ indicate
	$\Aut(w_i)=\nn(w_i)\wedge\neg\mathrm{Stab}_\delta(\Aut,w_i)$. These are
	i.i.d. Bernoulli variables of mean $\lambda$. Checker completeness implies
	that acceptance entails $Z_i=0$ for every $i$. If
	$\lambda>1-\eta^{1/n}$, the probability of the latter event is
	$(1-\lambda)^n<\eta$. Hence the probability that the phase accepts while
	$\lambda>1-\eta^{1/n}$ is below $\eta$, which proves the claim.
\end{proof}

Combining Lemmata~\ref{thm:approx-eq} and~\ref{thm:robust-matching-positives}
with a union bound yields the following theorem.

\begin{theorem}[End-to-end PAC guarantee for an accepted DRA]\label{thm:end-to-end}
	Define
	$
	\mathrm{Good}_\delta(\nn,\Aut,w) \coloneqq
	(\Aut(w)=\nn(w)) \wedge \mathrm{Stab}_\delta(\Aut,w).
	$
	Let $n=\lceil \ln(1/\varepsilon)/(2\gamma^2)\rceil$ be the sample size used by
	the validation procedure. For any $\eta\in(0,1)$, if it returns
	an accepted DRA $\Aut$, then with confidence at least $1-\varepsilon-\eta$,
	\[
	\Pr\nolimits_{w\sim\mathcal D}\bigl[\mathrm{Good}_\delta(\nn,\Aut,w)\bigr]
	\ge p-\gamma-\bigl(1-\eta^{1/n}\bigr).
	\]
\end{theorem}
\begin{proof}
	By definition, with all probabilities taken over $w\sim\mathcal D$,
	\[
	\Pr[\mathrm{Good}_\delta(\nn,\Aut,w)]
	=\Pr[\Aut(w)=\nn(w)]
	-\Pr[\Aut(w)=\nn(w)\wedge\neg\mathrm{Stab}_\delta(\Aut,w)].
	\]
	With confidence at least $1-\varepsilon$, Lemma~\ref{thm:approx-eq} lower-bounds
	the first term by $p-\gamma$; with confidence at least $1-\eta$,
	Lemma~\ref{thm:robust-matching-positives} bounds the second term above by
	$1-\eta^{1/n}$. A union bound yields the stated result.
\end{proof}

%\paragraph{Scope of the guarantee.}
%The validation loop distinguishes three cases at an agreeing sample $w$:
%(i)~$\Aut$ flips at some $w'\in B_\delta(w)$ and $\nn$ flips too, which
%yields a genuine witness of $\nn$'s non-robustness; (ii)~$\Aut$ flips at
%$w'$ but $\nn$ does not, a spurious counterexample that refines the
%learner; and (iii)~$\Aut$ is stable on $B_\delta(w)$. In case~(iii), $\nn$
%may still flip at some $w''\in B_\delta(w)$; then $\Aut$ and $\nn$ disagree
%on $w''$, but the loop does not detect this, because the checker searches
%for flips of $\Aut$ and $\nn$ is only queried as a black box.
%\prettyref{thm:end-to-end} is consistent with this: $\mathrm{Good}_\delta$
%asserts stability of the surrogate $\Aut$ and agreement at $w$, not
%stability of $\nn$ on $B_\delta(w)$. Disagreements at points such as $w''$
%are covered by \prettyref{thm:approx-eq} only to the extent that they carry
%mass under $\mathcal D$; neighbourhood points are typically off-distribution,
%so an accepted DRA certifies that $\nn$ is robust at $w$ only up to its
%fidelity on $B_\delta(w)$. Without white-box access, case~(iii) can only be
%detected by testing, e.g.\ by querying $\nn$ on points of $B_\delta(w)$ and
%feeding disagreements back as refinement samples; this yields evidence but no
%certificate of $\nn$'s robustness, and we leave it to future work.

\begin{remark}
	The theorem is a guarantee about $\Aut$, not $\nn$: $\mathrm{Stab}_\delta(\Aut,w)$
	holds for the \emph{surrogate}, but $\nn$ may still flip somewhere in
	$B_\delta(w)$ where $\Aut$ does not. It also concerns one validation
	phase, whereas extraction performs up to $B$ of them and returns the first
	acceptance; since each phase draws its $n$ samples afresh, independent of
	the candidate, a union bound over the $B$ phases gives
	\prettyref{thm:end-to-end} at confidence $1-B(\varepsilon+\eta)$, restored
	to the stated confidence by running each phase at $\varepsilon/B,\eta/B$,
	costing only $\ln B/(2\gamma^2)$ extra samples per phase.
\end{remark}

\section{Experimental Evaluation}\label{sec:exp}
We evaluate \toolName~\cite{hong_2026_23451085},  which extracts deterministic register automata (DRAs) as
interpretable surrogates for neural sequence classifiers. It combines local-search
DRA synthesis with fixed-register robustness checking under last-letter, edit,
Hamming, and Manhattan, and DTW distance through the RAA construction.
We focus on the following research questions:
\begin{itemize}[leftmargin=*]
	\item \textbf{(RQ1)} Do the benchmark languages admit compact DRAs, and are
	the networks accurate on them?
	\item \textbf{(RQ2)} How effective is the PAC validation loop at accepting
	surrogates or finding robustness violations?
	\item \textbf{(RQ3)} How do learning and checking scale with the numbers
	of states and registers and with the input length?
\end{itemize}

\paragraph{Benchmarks.}
\prettyref{tab:bench} lists the benchmarks. $L_1$--$L_7$ are the
Tomita-style languages~\cite{WGY18} with $a$ and $b$ encoded as distinct data values,
including counting modulo a constant ($L_5$, $L_6$). $S_1$--$S_{11}$ are
idealised trend patterns from technical
analysis~\cite{murphy1999technical}, whose membership depends on comparing
values at different positions; the learner is given one ($S_1$--$S_8$) or
two ($S_9$--$S_{11}$) registers. \textsf{MNIST} reads each image as the sequence
of its row-mean intensities and has no ground-truth DRA.

\begin{table}[htbp]
	\caption{%
		Benchmarks and RQ1 results. $L_i$: Tomita languages; $S_i$:
		trend patterns over rational sequences, incl.\ the
		higher-highs/higher-lows family of \prettyref{eg:uptrend}
		($S_9$--$S_{11}$). \textit{Acc}: balanced test accuracy (\%); \textit{Time}: synthesis time (s);
		\textit{S/R/T}: states/registers/transitions.%
	}\label{tab:learn-res}\label{tab:bench}
	\centering \scriptsize \renewcommand{\arraystretch}{0.74} \setlength{\tabcolsep}{3.0pt}
%	\scalebox{.95}{
		\begin{tabular}{l p{5.4cm} cc|lll}
			\toprule
			& & \multicolumn{2}{c|}{\nnLSTM\hspace{1.2em}\nnTransformer}
			& \multicolumn{3}{c}{\textbf{\learnLS}}   \\[1ex]
			\textbf{Language} & \textbf{Description}
			& \textbf{Acc} & \textbf{Acc}
			& \textbf{Acc} & \textbf{Time} &  \textbf{S R T}  \\
			\midrule
			$L_1$ & $a^*$ & 100.00 & 100.00 & 100.00 & 1.00 & 2 1 2 \\[1ex]
			$L_2$ & $(ab)^*$ & 100.00 & 100.00 & 99.95 & 20.08 & 4 1 8 \\[1ex]
			$L_3$ & $a^n b^m$ with odd $n$ and even $m$ & 100.00 & 100.00 & 100.00 & 0.53 & 4 1 4 \\[1ex]
			$L_4$ & No symbol repeats three times consecutively & 97.90 & 99.30 & 99.45 & 27.90 & 4 2 8 \\[1ex]
			$L_5$ & $a(a|b)^*$ with even $\mathrm{count}_a,\mathrm{count}_b$ & 100.00 & 100.00 & 79.00 & 753.26 & 4 1 7 \\[1ex]
			$L_6$ & $a(a|b)^*$ with $\mathrm{count}_a \equiv_3 \mathrm{count}_b$ & 100.00 & 100.00 & 100.00 & 11.72 & 3 2 6 \\[1ex]
			$L_7$ & $a^+ b^* a^* b^*$ & 100.00 & 98.20 & 100.00 & 335.81 & 5 1 10 \\[1ex]
			$S_1$ & Strictly increasing & 99.80 & 99.65 & 100.00 & 2.06 & 2 1 3 \\[1ex]
			$S_2$ & Strictly decreasing & 99.80 & 99.30 & 100.00 & 39.85 & 2 1 2 \\[1ex]
			$S_3$ & Non-strictly increasing & 100.00 & 100.00 & 100.00 & 0.30 & 2 1 3 \\[1ex]
			$S_4$ & Non-strictly decreasing & 99.90 & 100.00 & 100.00 & 2.66 & 2 1 4 \\[1ex]
			$S_5$ & Single peak & 100.00 & 99.35 & 100.00 & 353.90 & 4 1 8 \\[1ex]
			$S_6$ & Single valley & 100.00 & 98.30 & 100.00 & 310.51 & 4 1 7 \\[1ex]
			$S_7$ & Two peaks & 99.95 & 97.75 & 99.80 & 262.98 & 5 1 10 \\[1ex]
			$S_8$ & Three peaks & 100.00 & 99.65 & 100.00 & 83.02 & 6 1 12 \\[1ex]
			$S_9$ & Higher highs/lows & 99.65 & 99.80 & 99.90 & 543.55 & 5 2 10 \\[1ex]
			$S_{10}$ & Higher highs, lower lows & 99.55 & 98.75 & 99.60 & 504.28 & 5 2 10 \\[1ex]
			$S_{11}$ & Lower highs/lows & 100.00 & 100.00 & 99.40 & 387.35 & 5 2 10 \\[1ex]
			\textsf{MNIST} & Row-mean pixel intensity, digit 1 vs.\ rest & 98.63 & 98.50 & 95.37 & 832.09 & 4 1 7 \\
			\bottomrule
		\end{tabular}
%	}
\end{table}
\paragraph{Experimental Setup.}
We use $p=0.95$ and $\varepsilon=\gamma=0.05$ ($n=600$, $d_{\max}=30$),
strict neighbourhoods with $\delta=1$ ($\delta=1.5$ for Hamming and edit
distance), and per-configuration budgets of 30 minutes for learning and 60
minutes for checking. Further details, including neural network training and
the local-search hyperparameters, are given in \appref{app:setup}{C.1}.

\paragraph{RQ1: Language identifiability.}
%\learnLS\ learns from 369 positive and 369 negative sequences (length at most
%10; $\lceil\ln(2/\varepsilon)/(2\gamma^2)\rceil=738$ in total) of each language,
%independently of the networks.
\prettyref{tab:learn-res} shows that \learnLS\ finds DRAs with 2--6 states and
1--2 registers reaching at least 99.4\% on all synthetic benchmarks but
$L_5$ (79.0\%) and 95.37\% on \textsf{MNIST}.
We compare \learnLS\ with two alternatives in \appref{app:learners}{C.4}. An
SMT-based learner, given the same samples and budget, reaches 91.05--100\% on
$L_1$--$L_7$ and $S_1$--$S_8$, as it stops at 90--95\% training accuracy, is
slower on most of them (median 82~s vs.\ 28~s), and times out on the
two-register patterns $S_9$--$S_{11}$. The active learner
RALib~\cite{dierl2024scalable} instead queries the language: it learns
$L_1$--$L_7$ and $S_1$--$S_8$ exactly within 30~s, but times out or fails on
$S_9$--$S_{11}$ and cannot run on \textsf{MNIST}, which has no membership
oracle. Thus only \learnLS\ handles all benchmarks. Both networks
reach at least 97.75\% on every synthetic benchmark and 98.5\% on
\textsf{MNIST}, so RQ2 is not confounded by poorly trained black boxes.

\paragraph{RQ2: End-to-end robustness checking.}
For each of the 190 combinations of benchmark, network and metric, we run the
full validate--check--refine loop of \prettyref{sec:validation-extraction},
not a single checker call. A run ends \emph{robust} when a surrogate is
PAC-accepted, \emph{non-robust} when a witness also flips the network, and
\emph{inconclusive} after 10 rounds otherwise. 180 runs concluded (38 robust,
142 non-robust, no timeouts) with a median of 186~s and at most 1659~s;
learning took 92.0\% of the runtime and checking 2.8\%; cf. \appref{app:rq2-outcomes}{C.5}.

\begin{table}[hbpt]
	\caption{Validation loop per metric (38 configurations each).
	\textit{Multi}/\textit{Max}: concluded runs with several rounds / most
	rounds; \textit{Spur.}: spurious witnesses in concluded runs; \textit{Cost}:
	confirmed witness costs; $|w|\neq|v|$: witnesses changing the length.}
	\label{tab:loop}
	\centering\scriptsize
	\setlength{\tabcolsep}{4pt}\renewcommand{\arraystretch}{0.85}
	\begin{tabular}{l c cc r c c}
		\toprule
		Metric & $+/-/?$ & Multi & Max & Spur. & Cost & $|w|\neq|v|$\\
		\midrule
		Last letter & 21/14/3 & 8 & 5 & 358 & 0.0625 & 0\\
		Edit & 0/36/2 & 6 & 6 & 362 & 1 & 5\\
		Hamming & 2/35/1 & 5 & 5 & 50 & 1 & 0\\
		Manhattan & 15/19/4 & 11 & 9 & 835 & 0.011--0.106 & 0\\
		DTW & 0/38/0 & 6 & 5 & 232 & 0--0.083 & 36\\
		\bottomrule
	\end{tabular}
\end{table}

\prettyref{tab:loop} details the loop. 36 concluded runs needed refinement,
which discarded 1{,}837 spurious witnesses that flip the DRA but not the
network. Accepted surrogates agree with the network on at least 95.2\%
(median 99.6\%) of the validation samples. Of the 10 inconclusive runs, seven
still exceeded $d_{\max}$ after 10 rounds and three kept producing spurious
witnesses.

Robustness holds mainly under in-place perturbations: last-letter and
Manhattan distance account for 36 of the 38 accepted surrogates, and their
witnesses cost at most $0.11$. One substitution or edit flips the network in
35 and 36 of 38 configurations, respectively. DTW exposes non-robustness in
all 38: 26 witnesses cost~0, so realignment alone flips the label, and 36
change the input length, which fixed-length checkers cannot detect.

\paragraph{RQ3: Scalability.}
We use three families: strictly increasing sequences of length $n$ (one
register, $n+1$ states, $n\in\{20,40,\ldots,100\}$); length-$n$ sequences of
pairwise distinct values ($n$ registers, $n+2$ states, $n\in\{2,4,6\}$); and
$S_1$ with maximum training length 20--100. For each point, \learnLS\ learns
a DRA of the stated size, which we test on 1{,}000 and check on 100 fresh
sequences under last-letter distance ($\delta=1$). 

As shown in \prettyref{fig:RQ3_1}, learning is the bottleneck. For increasing
sequences, it exhausts the budget from 20 states on, yet reaches 94.7--98.4\%
test accuracy; for distinct values, its time grows from 6~s to 1029~s between
2 and 6 registers; input length hardly matters (every $S_1$ point within 9~s,
at least 99.5\% accuracy). Checking stays cheap, at most 11~s per sequence:
its time grows roughly linearly with the number of states (0.27--3.37~s),
mildly with input length (0.08--0.43~s), and not with registers (0.01~s).

\begin{figure}[hbpt]
	\centering
	\begin{tikzpicture}
		\begin{groupplot}[
			group style={group size=3 by 2, horizontal sep=0.3cm,
				vertical sep=0.4cm, xticklabels at=edge bottom,
				yticklabels at=edge left},
			width=0.4\textwidth, height=3.2cm,
			ymode=log, log basis y=10, log ticks with fixed point,
			yminorticks=false, ylabel style={align=center},
			ymajorgrids=true, grid style={densely dotted, gray!40},
			tick label style={font=\scriptsize}, label style={font=\scriptsize},
			every axis plot/.append style={thick, mark size=1.6pt},
			]
			% learning time; hollow marks: run hit the 30-minute budget
			\nextgroupplot[ylabel={Learning\\(s)}, ymin=0.1, ymax=5000, ytick={0.1,1,10,100,1000},
			xmin=10, xmax=110, xtick={20,40,60,80,100}]
			\addplot[gray, dashed, thin, mark=none, domain=10:110] {1800};
			\addplot[lstm_rob, mark=o] coordinates {
				(20,1800.36) (40,1800.39) (60,1801.61) (80,1800.44) (100,1800.51)};
			\nextgroupplot[ymin=0.1, ymax=5000, xmin=1, xmax=7, xtick={2,4,6}]
			\addplot[gray, dashed, thin, mark=none, domain=1:7] {1800};
			\addplot[lstm_rob, mark=*] coordinates {(2,6.11) (4,608.33) (6,1028.77)};
			\nextgroupplot[ymin=0.1, ymax=5000, xmin=10, xmax=110,
			xtick={20,40,60,80,100}]
			\addplot[gray, dashed, thin, mark=none, domain=10:110] {1800};
			\addplot[lstm_rob, mark=*] coordinates {
				(20,8.71) (40,0.66) (60,0.38) (80,0.54) (100,0.58)};
			% average checking time per sequence
			\nextgroupplot[ylabel={Checking\\(s/sequence)}, xlabel={States},
			ymin=0.003, ymax=10, ytick={0.01,0.1,1,10},
			xmin=10, xmax=110, xtick={20,40,60,80,100}]
			\addplot[trans_rob, mark=square*] coordinates {
				(20,0.2749) (40,0.7593) (60,1.7408) (80,2.6464) (100,3.3713)};
			\nextgroupplot[xlabel={Registers}, ymin=0.003, ymax=10, xmin=1, xmax=7,
			xtick={2,4,6}]
			\addplot[trans_rob, mark=square*] coordinates {
				(2,0.0106) (4,0.0092) (6,0.0075)};
			\nextgroupplot[xlabel={Max.\ sequence length}, ymin=0.003, ymax=10,
			xmin=10, xmax=110, xtick={20,40,60,80,100}]
			\addplot[trans_rob, mark=square*] coordinates {
				(20,0.1333) (40,0.0762) (60,0.1561) (80,0.3395) (100,0.4284)};
		\end{groupplot}
	\end{tikzpicture}
	\caption{Learning time and checking time per sequence under
	last-letter distance.}
%	Dashed: 30-minute learning budget; hollow
%	marks: runs that hit it.}
	\label{fig:RQ3_1}\label{fig:RQ3_2}
\end{figure}

\section{Related Work and Conclusion}\label{sec:related-work}\label{sec:conclusion}

Automata learning underlies grammatical inference, from graph-colouring and SAT-based DFA identification~\cite{higuera-book,heule2010exact} to EDSM-style state merging~\cite{HV13,Lang98}. Related register-automata learners~\cite{dierl2024scalable,learning-nominal-automata,DBLP:conf/concur/BalachanderFG24} do not learn constants in transition guards, which our benchmarks require. Formal explainability studies complexity and constraint-based explanations~\cite{barcelo2020model,SN22}, while extraction methods derive DFAs, weighted automata, or Moore machines from RNNs and transformers~\cite{6796344,WGY18,ANGLUIN198787,OWSH20,wei2022extractingweightedfiniteautomata,AM25}. These use finite alphabets; we learn DRAs over data domains. For robustness, solver and abstraction methods~\cite{Katz17,AI2,ko2019popqorn,du2021certRNN,jacoby2020verifying,mohammadinejad2020diffrnn,zhang2023rnnguard,lin2026rnnrobustness,shi2020transformer,liao2023transformers,hong2025transformer,huang2026fairness} are mostly white-box, whereas probabilistic approaches~\cite{DBLP:journals/sttt/KhmelnitskyNRXBBFHLY23} match our black-box validation setting. Our counterexamples are minimum-change contrastive explanations, in contrast to subset-minimal abductive or contrastive explanations~\cite{Ignatiev2019Abduction,Ignatiev2020Contrastive,barcelo2020model}.

Beyond DRAs, symbolic register automata~\cite{SRA}, data automata~\cite{two-variable-logic,FL22}, and variable automata~\cite{FL22,variable-automata} suggest richer future surrogates, for instance for LRA guards in financial indicators~\cite{murphy1999technical}. In summary, we give a polynomial-time robustness checker for DRAs with a fixed number of registers and combine it with local-search learning and PAC validation into a black-box pipeline that extracts compact surrogates and non-robustness witnesses for sequence classifiers.

Several limitations remain. The polynomial bound assumes a fixed number of registers; scaling to many registers needs better symbolic reductions. The validation guarantee is distributional: an accepted DRA is faithful on the reference distribution, not on all sequences. The learner is heuristic and may fail when the network has no compact DRA approximation. Natural future work includes probabilistic and nondeterministic surrogates, richer LRA-based guards, and properties such as fairness, monotonicity, and temporal specifications.

\begin{credits}
\subsubsection{\ackname}
Jiang, Lin, Markgraf, and Parsert are supported by the
European Research Council (ERC) under the European Union’s Horizon 2020
research and innovation programme under grant number 101089343. Chih-Duo Hong is supported by the National Science and Technology Council, Taiwan,
under grant numbers 112-2222-E004-001-MY3 and 114-2634-F-004-002-MBK.
We thank the anonymous SETTA~2026 reviewers for their helpful comments and suggestions.
\end{credits}

\bibliographystyle{splncs04}
\bibliography{literature}

@article{Angluin1988Queries,
  author = {Dana Angluin},
  title = {Queries and Concept Learning},
  journal = {Mach. Learn.},
  year = {1988},
  volume = {2},
  pages = {319--342},
}

@software{hong_2026_23451085,
	author       = {Hong, Chih-Duo and
	Jiang, Hongjian and
	Lin, Anthony Widjaja and
	Markgraf, Oliver and
	Parsert, Julian and
	Tan, Tony},
	title        = {Synthesis and Verification of Transformer Programs},
	month        = oct,
	year         = 2026,
	publisher    = {Zenodo},
	doi          = {10.5281/zenodo.23451085},
	url          = {https://doi.org/10.5281/zenodo.23451085},
}

@inproceedings{ribeiro2016why,
  author = {Marco Tulio Ribeiro and
	Sameer Singh and
	Carlos Guestrin},
  title = {"Why Should I Trust You?": Explaining the Predictions of Any Classifier},
  booktitle = {{KDD}},
  year = {2016},
  pages = {1135--1144},
}

@inproceedings{ribeiro2018anchors,
  author = {Marco Tulio Ribeiro and
	Sameer Singh and
	Carlos Guestrin},
  title = {Anchors: High-Precision Model-Agnostic Explanations},
  booktitle = {{AAAI}},
  year = {2018},
  pages = {1527--1535},
}

@inproceedings{lundberg2017unified,
  author = {Scott M. Lundberg and
	Su{-}In Lee},
  title = {A Unified Approach to Interpreting Model Predictions},
  booktitle = {{NeurIPS}},
  year = {2017},
  pages = {4765--4774},
}

@inproceedings{Ignatiev2019Abduction,
  author = {Alexey Ignatiev and Nina Narodytska and Jo{\~a}o Marques{-}Silva},
  title = {Abduction-Based Explanations for Machine Learning Models},
  booktitle = {{AAAI}},
  year = {2019},
  pages = {1511--1519},
}

@inproceedings{Ignatiev2020Contrastive,
  author = {Alexey Ignatiev and Nina Narodytska and Jo{\~a}o Marques{-}Silva},
  title = {From Contrastive to Abductive Explanations and Back Again},
  booktitle = {AIxIA},
  year = {2021},
  pages = {335--355},
}

@book{kearns1994introduction,
  author = {Kearns, Michael J and Vazirani, Umesh},
  title = {An introduction to computational learning theory},
  year = {1994},
  publisher = {MIT press},
}

@article{hong2025transformer,
  author = {Chih-Duo Hong and Chih-Cheng Yang and Yu Wang and Fang Yu},
  title = {Influence-Guided Concolic Testing of Transformer Robustness},
  journal = {arXiv:2509.23806},
  year = {2025},
  eprint = {2509.23806},
  doi = {10.48550/arXiv.2509.23806}
}

@article{huang2026fairness,
  author = {Ming-I Huang and Chih-Duo Hong and Fang Yu},
  title = {Concolic Testing on Individual Fairness of Neural Network Models},
  year = {2026},
  journal = {arXiv:2509.06864},
  eprint = {2509.06864},
}

@article{lin2026rnnrobustness,
  author = {Li-Jen Lin and Chih-Duo Hong},
  title = {Robustness Verification of Recurrent Neural Networks with Abstraction Refinement},
  year = {2026},
  journal = {arXiv:2606.12490},
  eprint = {2606.12490},
}

@article{10.1162/neco.1997.9.8.1735,
	author = {Hochreiter, Sepp and Schmidhuber, J\"{u}rgen},
	title = {Long Short-Term Memory},
	year = {1997},
	issue_date = {November 15, 1997},
	publisher = {MIT Press},
	address = {Cambridge, MA, USA},
	volume = {9},
	number = {8},
	issn = {0899-7667},
	url = {https://doi.org/10.1162/neco.1997.9.8.1735},
	doi = {10.1162/neco.1997.9.8.1735},
	journal = {Neural Comput.},
	month = nov,
	pages = {1735–1780},
	numpages = {46}
}

@article{6796344,
  author = {Giles, C. L. and Miller, C. B. and Chen, D. and Chen, H. H. and Sun, G. Z. and Lee, Y. C.},
  title = {Learning and Extracting Finite State Automata with Second-Order Recurrent Neural Networks},
  journal = {Neural Comput.},
  year = {1992},
  volume = {4},
  number = {3},
  pages = {393-405},
}

@inproceedings{DBLP:conf/concur/KhalimovK19,
  author = {Ayrat Khalimov and
	Orna Kupferman},
  title = {Register-Bounded Synthesis},
  booktitle = {{CONCUR}},
  year = {2019},
  pages = {25:1--25:16},
}

@inproceedings{DBLP:conf/fsttcs/AbdullaAKR14,
  author = {Parosh Aziz Abdulla and
	Mohamed Faouzi Atig and
	Ahmet Kara and
	Othmane Rezine},
  title = {Verification of Dynamic Register Automata},
  booktitle = {{FSTTCS}},
  year = {2014},
  pages = {653--665},
}

@inproceedings{DBLP:conf/concur/BalachanderFG24,
  author = {Mrudula Balachander and
	Emmanuel Filiot and
	Raffaella Gentilini},
  title = {Passive Learning of Regular Data Languages in Polynomial Time and
	Data},
  booktitle = {{CONCUR}},
  year = {2024},
  pages = {10:1--10:21},
}

@article{wei2022extractingweightedfiniteautomata,
  author = {Zeming Wei and Xiyue Zhang and Meng Sun},
  title = {Extracting Weighted Finite Automata from Recurrent Neural Networks for Natural Languages},
  year = {2022},
  journal = {arXiv:2206.14621},
  eprint = {2206.14621},
}

@article{DemriL09,
  author = {St{\'{e}}phane Demri and
                  Ranko Lazic},
  title = {{LTL} with the freeze quantifier and register automata},
  journal = {{ACM} Trans. Comput. Log.},
  year = {2009},
  volume = {10},
  number = {3},
  pages = {16:1--16:30},
}

@InProceedings{10.1007/978-3-319-10431-7_18,
	author="Cassel, Sofia
	and Howar, Falk
	and Jonsson, Bengt
	and Steffen, Bernhard",
	editor="Giannakopoulou, Dimitra
	and Sala{\"u}n, Gwen",
	title="Learning Extended Finite State Machines",
	booktitle="Software Engineering and Formal Methods",
	year="2014",
	publisher="Springer International Publishing",
	address="Cham",
	pages="250--264",
	isbn="978-3-319-10431-7"
}

@inproceedings{alur2013regular,
  author = {Alur, Rajeev and D'Antoni, Loris and Deshmukh, Jyotirmoy and Raghothaman, Mukund and Yuan, Yifei},
  title = {Regular functions and cost register automata},
  booktitle = {{LICS}},
  year = {2013},
  pages = {13--22},
}

@inproceedings{dierl2024scalable,
  author = {Dierl, Simon and Fiterau-Brostean, Paul and Howar, Falk and Jonsson, Bengt and Sagonas, Konstantinos and T{\aa}quist, Fredrik},
  title = {Scalable Tree-based Register Automata Learning},
  booktitle = {{TACAS}},
  year = {2024},
  pages = {87--108},
}

@inproceedings{chen2017register,
  author = {Chen, Yu-Fang and Leng{\'a}l, Ond{\v{r}}ej and Tan, Tony and Wu, Zhilin},
  title = {Register automata with linear arithmetic},
  booktitle = {{LICS}},
  year = {2017},
  pages = {1--12},
}

@article{barcelo2020model,
  author = {Barcel{\'o}, Pablo and Monet, Mika{\"e}l and P{\'e}rez, Jorge and Subercaseaux, Bernardo},
  title = {Model interpretability through the lens of computational complexity},
  journal = {{NeurIPS}},
  year = {2020},
  volume = {33},
  pages = {15487--15498},
}

@inproceedings{SN22,
  author = {Aditya A. Shrotri and
                  Nina Narodytska and
                  Alexey Ignatiev and
                  Kuldeep S. Meel and
                  Jo{\~{a}}o Marques{-}Silva and
                  Moshe Y. Vardi},
  title = {Constraint-Driven Explanations for Black-Box {ML} Models},
  booktitle = {{AAAI}},
  year = {2022},
  pages = {8304--8314},
}

@inproceedings{WGY18,
  author = {Gail Weiss and
                  Yoav Goldberg and
                  Eran Yahav},
  title = {Extracting Automata from Recurrent Neural Networks Using Queries and
                  Counterexamples},
  booktitle = {{ICML}},
  year = {2018},
  pages = {5244--5253},
}

@inproceedings{learning-nominal-automata,
  author = {Joshua Moerman and
                  Matteo Sammartino and
                  Alexandra Silva and
                  Bartek Klin and
                  Michal Szynwelski},
  title = {Learning nominal automata},
  booktitle = {{POPL}},
  year = {2017},
  pages = {613--625},
}

@article{vaswani2017attention,
  author = {Vaswani, Ashish and Shazeer, Noam and Parmar, Niki and Uszkoreit, Jakob and Jones, Llion and Gomez, Aidan N and Kaiser, {\L}ukasz and Polosukhin, Illia},
  title = {Attention is all you need},
  journal = {{NeurIPS}},
  year = {2017},
  volume = {30},
}

@book{time-series-book,
    author = {Aileen Nielsen},
    title = {Practical Time Series Analysis: Prediction with Statistics and
        Machine Learning},
    year = {2019},
    publisher = {O'Reilly}
}

@inproceedings{OWSH20,
  author = {Takamasa Okudono and
                  Masaki Waga and
                  Taro Sekiyama and
                  Ichiro Hasuo},
  title = {Weighted Automata Extraction from Recurrent Neural Networks via Regression
                  on State Spaces},
  booktitle = {{AAAI}},
  year = {2020},
  pages = {5306--5314},
}

@misc{atom-book,
  author = {Mikolaj Bojanczyk},
  title = {Slightly Infinite Sets},
  year = {2019},
  howpublished = {\url{www.mimuw.edu.pl/~bojan/upload/main-10.pdf}},
}

@inproceedings{neider2013regular,
	title={Regular model checking using solver technologies and automata learning},
	author={Neider, Daniel and Jansen, Nils},
	booktitle={NASA Formal Methods Symposium},
	pages={16--31},
	year={2013},
	organization={Springer}
}

@inproceedings{heule2010exact,
  author = {Heule, Marijn JH and Verwer, Sicco},
  title = {Exact DFA identification using SAT solvers},
  booktitle = {{ICGI}},
  year = {2010},
  pages = {66--79},
}

@article{ANGLUIN198787,
  author = {Dana Angluin},
  title = {Learning regular sets from queries and counterexamples},
  journal = {Inf. Comput.},
  year = {1987},
  volume = {75},
  number = {2},
  pages = {87-106},
}

@inproceedings{time-series-survey,
  author = {Qingsong Wen and
                  Tian Zhou and
                  Chaoli Zhang and
                  Weiqi Chen and
                  Ziqing Ma and
                  Junchi Yan and
                  Liang Sun},
  title = {Transformers in Time Series: {A} Survey},
  booktitle = {{IJCAI}},
  year = {2023},
  pages = {6778--6786},
}

@inproceedings{DBLP:conf/lics/BojanczykMSSD06,
  author = {Mikolaj Bojanczyk and
	Anca Muscholl and
	Thomas Schwentick and
	Luc Segoufin and
	Claire David},
  title = {Two-Variable Logic on Words with Data},
  booktitle = {{LICS}},
  year = {2006},
  pages = {7--16},
}

@article{BKL14,
  author = {Mikolaj Bojanczyk and
                  Bartek Klin and
                  Slawomir Lasota},
  title = {Automata theory in nominal sets},
  journal = {Log. Methods Comput. Sci.},
  year = {2014},
  volume = {10},
  number = {3},
}

@article{Katz17,
  author = {Guy Katz and
                  Clark W. Barrett and
                  David L. Dill and
                  Kyle Julian and
                  Mykel J. Kochenderfer},
  title = {Reluplex: An Efficient {SMT} Solver for Verifying Deep Neural Networks},
  journal = {arXiv:1702.01135},
  year = {2017},
  eprint = {1702.01135},
}

@inproceedings{AI2,
  author = {Timon Gehr and
                  Matthew Mirman and
                  Dana Drachsler{-}Cohen and
                  Petar Tsankov and
                  Swarat Chaudhuri and
                  Martin T. Vechev},
  title = {{AI2:} Safety and Robustness Certification of Neural Networks with
                  Abstract Interpretation},
  booktitle = {{IEEE} S\&P},
  year = {2018},
  pages = {3--18},
}

@article{merrill2022extracting,
  author = {Merrill, William and Tsilivis, Nikolaos},
  title = {Extracting finite automata from RNNs using state merging},
  journal = {arXiv:2201.12451},
  year = {2022},
}

@inproceedings{SRA,
  author = {Loris D'Antoni and
                  Tiago Ferreira and
                  Matteo Sammartino and
                  Alexandra Silva},
  title = {Symbolic Register Automata},
  booktitle = {{CAV}},
  year = {2019},
  pages = {3--21},
}

@article{HV13,
  author = {Marijn Heule and
                  Sicco Verwer},
  title = {Software model synthesis using satisfiability solvers},
  journal = {Empir. Softw. Eng.},
  year = {2013},
  volume = {18},
  number = {4},
  pages = {825--856},
}

@book{murphy1999technical,
  author = {Murphy, John J.},
  title = {Technical analysis of the financial markets},
  year = 1999,
  publisher = {New York Institute of Finance},
  refid = {423808179},
}

@article{two-variable-logic,
  author = {Mikolaj Bojanczyk and
                  Claire David and
                  Anca Muscholl and
                  Thomas Schwentick and
                  Luc Segoufin},
  title = {Two-variable logic on data words},
  journal = {{ACM} Trans. Comput. Log.},
  year = {2011},
  volume = {12},
  number = {4},
  pages = {27:1--27:26},
}

@inproceedings{FL22,
  author = {Diego Figueira and
                  Anthony Widjaja Lin},
  title = {Reasoning on Data Words over Numeric Domains},
  booktitle = {{LICS}},
  year = {2022},
  pages = {37:1--37:13},
}

@inproceedings{variable-automata,
  author = {Orna Grumberg and
                  Orna Kupferman and
                  Sarai Sheinvald},
  title = {Variable Automata over Infinite Alphabets},
  booktitle = {{LATA}},
  year = {2010},
  pages = {561--572},
}

@book{higuera-book,
  author = {de la Higuera, Colin},
  title = {Grammatical Inference: Learning Automata and Grammars},
  year = {2010},
  publisher = {Cambridge University Press},
}

@inproceedings{Lang98,
  author = "Lang, Kevin J.
and Pearlmutter, Barak A.
and Price, Rodney A.",
  title = "Results of the Abbadingo one DFA learning competition and a new evidence-driven state merging algorithm",
  booktitle = {Grammatical Inference},
  year = "1998",
  pages = "1--12",
}

@inproceedings{AM25,
  author = {Rik Adriaensen and
                  Jaron Maene},
  title = {Extracting Moore Machines from Transformers Using Queries and Counterexamples},
  booktitle = {{IDA}},
  year = {2025},
  pages = {419--431},
}

@inproceedings{zhang21,
  author = {Xiyue Zhang and
                  Xiaoning Du and
                  Xiaofei Xie and
                  Lei Ma and
                  Yang Liu and
                  Meng Sun},
  title = {Decision-Guided Weighted Automata Extraction from Recurrent Neural
                  Networks},
  booktitle = {{AAAI}},
  year = {2021},
  pages = {11699--11707},
}

@article{DBLP:journals/sttt/KhmelnitskyNRXBBFHLY23,
  author = {Igor Khmelnitsky and
                  Daniel Neider and
                  Rajarshi Roy and
                  Xuan Xie and
                  Beno{\^{\i}}t Barbot and
                  Benedikt Bollig and
                  Alain Finkel and
                  Serge Haddad and
                  Martin Leucker and
                  Lina Ye},
  title = {Analysis of recurrent neural networks via property-directed verification
                  of surrogate models},
  journal = {STTT},
  year = {2023},
  volume = {25},
  number = {3},
  pages = {341--354},
}

@inproceedings{ko2019popqorn,
  author = {Ching{-}Yun Ko and Zhaoyang Lyu and Tsui{-}Wei Weng and Luca Daniel and Ngai Wong and Dahua Lin},
  title = {{POPQORN}: Quantifying Robustness of Recurrent Neural Networks},
  booktitle = {{ICML}},
  pages = {3468--3477},
  year = {2019},
}

@inproceedings{du2021certRNN,
  author = {Tianyu Du and Shouling Ji and Lujia Shen and Yao Zhang and Jinfeng Li and Jie Shi and Chengfang Fang and Jianwei Yin and Raheem Beyah and Ting Wang},
  title = {{Cert-RNN}: Towards Certifying the Robustness of Recurrent Neural Networks},
  booktitle = {{CCS}},
  pages = {516--534},
  year = {2021},
  doi = {10.1145/3460120.3484538},
}

@inproceedings{shi2020transformer,
  author = {Zhouxing Shi and Huan Zhang and Kai{-}Wei Chang and Minlie Huang and Cho{-}Jui Hsieh},
  title = {Robustness Verification for Transformers},
  booktitle = {{ICLR}},
  year = {2020},
  eprint = {2002.06622},
  doi = {10.48550/arXiv.2002.06622},
}

@inproceedings{jacoby2020verifying,
  author = {Yuval Jacoby and Clark Barrett and Guy Katz},
  title = {Verifying Recurrent Neural Networks using Invariant Inference},
  booktitle = {{ATVA}},
  year = {2020},
}

@article{mohammadinejad2020diffrnn,
  author = {Sara Mohammadinejad and Brandon Paulsen and Chao Wang and Jyotirmoy V. Deshmukh},
  title = {{DiffRNN}: Differential Verification of Recurrent Neural Networks},
  journal = {arXiv:2007.10135},
  year = {2020},
  eprint = {2007.10135},
}

@article{zhang2023rnnguard,
  author = {Yunruo Zhang and Tianyu Du and Shouling Ji and Peng Tang and Shanqing Guo},
  title = {{RNN-Guard}: Certified Robustness Against Multi-frame Attacks for Recurrent Neural Networks},
  journal = {arXiv:2304.07980},
  year = {2023},
  eprint = {2304.07980},
}

@inproceedings{liao2023transformers,
  author = {Brian Hsuan{-}Cheng Liao and Chih{-}Hong Cheng and Hasan Esen and Alois Knoll},
  title = {Are Transformers More Robust? Towards Exact Robustness Verification for Transformers},
  booktitle = {{SafeComp}},
  year = {2023},
  doi = {10.1007/978-3-031-40923-3_8},
}

@misc{fullversion,
	author       = {Hong, Chih-Duo and Jiang, Hongjian and Lin, Anthony W. and
	                Markgraf, Oliver and Parsert, Julian and Tan, Tony},
	title        = {Synthesis and Verification of Register Automata over Data
	                Sequences (Extended Version)},
	year         = {2026},
	howpublished = {arXiv:XXXX.XXXXX},
}

\ifextended
\appendix
\setcounter{secnumdepth}{2} % number appendix subsections (C.1, ...)

%!TEX root = ../main.tex

\providecommand{\cost}{\textsf{cost}}

\section{RAAs for the Distance Metrics}
\label{app:metric-raas}

We give the RAAs for the five metrics used in \prettyref{sec:exp}: Hamming,
DTW, Manhattan, edit, and last-letter distance. Transitions are written
$(p,\varphi,\psi,\chi,q,\mov)$ as in \prettyref{sec:prelim}, with $\emptyset$
denoting the empty register assignment.

\subsection{Hamming Distance}

The following zero-register construction keeps the cost model concrete;
non-negativity of each increment follows from its guard. Use one
initial/accepting state and three self-loops moving both heads:
$\curr_1=\curr_2$ adds $0$, while $\curr_1<\curr_2$ and $\curr_1>\curr_2$ each
add $1$. Thus, accepting runs have equal-length words and exactly their
Hamming cost.

\subsection{Dynamic Time Warping}

Use an initial non-final state $q_s$ and final state $q$. For every
$\mov\in\{1,2,\mathsf{both}\}$, put the following transitions both from
$q_s$ to $q$ and as self-loops on $q$:
\[
\curr_1\leq\curr_2:\ \acc\pluseq\curr_2-\curr_1,
\qquad
\curr_1>\curr_2:\ \acc\pluseq\curr_1-\curr_2.
\]
Every run is an alignment of $v$ and $w$: a $\mov=1$ step matches the next
letter of $v$ with the current letter of $w$ again, and symmetrically for
$\mov=2$. Since the RAA is non-deterministic, minimising over runs minimises
over alignments, which yields DTW. Since either head may advance alone, the
witness word has no a priori length bound.

\subsection{Manhattan Distance}

For two sequences $v = v_1, \dots, v_n$ and $w = w_1, \dots, w_n$ of the same length,
their Manhattan distance is:
$$
D_{\textup{Man}}(v,w)\ := \ \sum_{1\leq i\leq n} |v_i-w_i|
$$
Similar to the Hamming RAA above,
we define the Manhattan distance between sequences of different length to be $\infty$.
The following RAA captures this metric.

\begin{center}
\begin{tikzpicture}[node distance=3cm]

\node[state,initial,accepting] (q0)   {$q_0$};

\draw[->] (q0) to [out=130,in=50,looseness=15] 
	node [above] {\footnotesize$\curr_1 \leq \curr_2,\emptyset,\acc \pluseq \curr_2-\curr_1, \textsf{both}$} (q0);

\draw[->] (q0) to [out=230,in=310,looseness=15] 
	node [below] {\footnotesize$\curr_1 > \curr_2,\emptyset,\acc \pluseq \curr_1-\curr_2, \textsf{both}$} (q0);

\end{tikzpicture}
\end{center}
Intuitively on input $(v,w)$, it reads each letter from $v$ and $w$ simultaneously.
When $\curr_2\geq \curr_1$, it increments $\acc$ by $\curr_2-\curr_1$.
Otherwise, it increments $\acc$ by $\curr_1-\curr_2$.

\subsection{Edit Distance}

The edit distance between two sequences 
$v = v_1 \cdots v_n$ and $w = w_1 \cdots w_m$ is the minimum number of operations 
(insertion, deletion and substitution of a single letter)
required to transform $v$ into $w$.
To define it formally, we need the following notation.
Let $M=((i_0,j_0),\dots, (i_\ell,j_\ell))$ be a sequence where 
each pair $(i_h,j_h) \in \{1,\ldots,n\}\times\{1,\ldots,m\}$.
We say that it is a \emph{monotone matching} (over $(n,m)$), if 
for every $h<h'$ we have $i_h< i_{h'}$ and $j_h < j_{h'}$.
Intuitively, such $M$ represents insertion/deletion/substitution as follows.
\begin{itemize}
\item
A pair $(i,j)$ that appears in $M$ represents a substitution of $v_i$ with $w_j$.
\item
An index $i\in \{1,\ldots,n\}$ that does not appear in $M$
represents the deletion of the letter $v_i$.
\item
An index $j\in \{1,\ldots,m\}$ that does not appear in $M$
represents the insertion of the letter $w_j$ into $v$.

Obviously, we can also view such index $j$ as representing 
the deletion of the letter $w_j$.
\end{itemize}
The edit distance between $v$ and $w$ w.r.t. $M$ is defined as:
\[
D_{\edt,M}(v,w) := \left(\sum_{(i,j)\in M} f(v_i,w_j) \right) + c_{\mathrm{insdel}}(n+m-2|M|),
\]
where $f(v_i,w_j)$ denotes the cost of substituting $v_i$ with $w_j$
and $c_{\mathrm{insdel}}$ is the cost of deleting all the letters that 
are not involved in the matching $M$.
We will consider the usual choice $f(a,a) = 0$, 
$f(a,b) = 1$ for $a\ne b$ and $c_{\mathrm{insdel}} = 1$, unless specified otherwise.
The \emph{edit distance} between $v$ and $w$ is defined as:
\[
D_{\edt}(v,w) := \min_{M}\ d_{\edt,M} (v,w)
\]
As an example, let $v=1,2,3,7,9$
and $w=1,3,7,10$.
For a monotone matching $M=\{(1,1),(3,2),(4,3)\}$
the edit distance between $v$ and $w$ w.r.t. $M$ is $3$.
By routine inspection, $D_{\edt}(v,w)=2$.

The RAA for $D_{\edt}$ is as follows.
\begin{center}
\begin{tikzpicture}[node distance=3cm]

\node[state,initial,accepting] (q0)   {$q_0$};

\draw[->] (q0) to [out=160,in=100,looseness=15] 
	node [above] {\footnotesize$t_1$} (q0);

\draw[->] (q0) to [out=200,in=260,looseness=15] 
	node [below] {\footnotesize$t_2$} (q0);

\draw[->] (q0) to [out=70,in=40,looseness=23] 
	node [right] {\footnotesize$t_3$} (q0);
	
\draw[->] (q0) to [out=20,in=340,looseness=20] 
	node [right] {\footnotesize$t_4$} (q0);
	
\draw[->] (q0) to [out=290,in=320,looseness=23] 
	node [right] {\footnotesize$t_5$} (q0);

\end{tikzpicture}
\end{center}
where:
\begin{itemize}
\item
$t_1$ is $(q_0,\curr_1 \neq\curr_2, \emptyset,\acc \pluseq 1,q_0,\textsf{both})$.
\item
$t_2$ is $(q_0,\curr_1 =\curr_2, \emptyset,\acc \pluseq 0,q_0,\textsf{both})$.
\item
$t_3$ is $(q_0,\top, \emptyset,\acc \pluseq 2,q_0,\textsf{both})$.
\item
$t_4$ is $(q_0,\top, \emptyset,\acc \pluseq 1,q_0,1)$.
\item
$t_5$ is $(q_0,\top, \emptyset,\acc \pluseq 1,q_0,2)$.
\end{itemize}
Intuitively, the RAA ``guesses'' the monotone matching $M$.
In applying each transition, it guesses whether the positions of both heads
should be in the monotone matching $M$.
If the guess is positive, i.e., the pair is in $M$, 
it applies either transition $t_1$ or $t_2$,
which corresponds to the case when $\curr_1\neq \curr_2$ or $\curr_1=\curr_2$.
It the negative,
it applies either transition $t_3$, $t_4$ or $t_5$,
which correspond to the case when both heads move,
only the first head moves or only the second head moves.
Transition $t_4$ reads neither letter and advances only the first head, so by
the semantics of \prettyref{sec:prelim} it does not need the second head; it
can therefore delete letters of $v$ after $w$ has been read completely, and
symmetrically $t_5$ can insert letters of $w$ after $v$ has been read
completely. Without this, the RAA could not edit the ends of the sequences
and would overestimate, e.g., the distance between $1,2,3$ and $1,2$.

\subsection{Last-Letter Distance}

The \emph{last-letter distance} between $v=v_1,\ldots,v_n$ and $w=w_1,\ldots,w_m$ is:
\[
D_{\textup{ll}}(v,w) := 
\left\{
\begin{array}{ll}
\infty & \text{if}\ m\neq n
\\
  \infty & \text{if}\ m= n \ \text{and there is}\ \\
  & \quad 1\leq i< n\ \text{where}\ v_i\neq w_i\\
|v_n-w_m|\quad & \text{if}\ m=n
\end{array}
\right.
\]
The RAA for this distance can be defined as follows.
\begin{center}
\begin{tikzpicture}[node distance=3cm]

\node[state,initial] (q0)   {$q_0$};

\node[state,accepting] (p)  [right=of q0]  {$p$};

\draw[->] (q0) to [out=130,in=50,looseness=15] 
	node [above] {\footnotesize$\curr_1 = \curr_2,\emptyset,\acc \pluseq 0, \textsf{both}$} (q0);

\draw[->] (q0) to [bend left] 
	node [above] {\footnotesize$t$} (p);

\draw[->] (q0) to [bend right] 
	node [below] {\footnotesize$t'$} (p);

\end{tikzpicture}
\end{center}
where:
\begin{itemize}
\item
$t$ is transition $(q_0,\curr_1 \geq \curr_2,\emptyset,\acc \pluseq \curr_1-\curr_2,p,\textsf{both})$,
\item
$t'$ is transition $(q_0,\curr_1 < \curr_2,\emptyset,\acc \pluseq \curr_2-\curr_1,p,\textsf{both})$.
\end{itemize} 
Intuitively, on input $(v,w)$ the RAA checks that they are equal up to the last letter.
If so, it increments the accumulator register $\acc$ by the absolute value of their difference.

%!TEX root = ../main.tex

\section{Graph Construction for Lemma~\ref{lem:cover-path}}
\label{app:cover-path-graph}

\prettyref{sec:robust} proves \prettyref{lem:cover-path} and describes the
graph informally. Here we restate the lemma and give the full construction,
in particular the pending second-head value and the tangent tags that handle
strict guards at boundary values. The proof uses \prettyref{lem:closure} and
\prettyref{lem:register-provenance} from \prettyref{sec:robust}.

\par\noindent\textbf{Lemma~\ref{lem:cover-path}} (Reducing coverability to shortest path)\textbf{.}
  \emph{Let $k$ be a fixed positive integer. Given an $\alpha$-bounded
  $v$-projected $k$-RAA $\Aut_D$ and a rational number $\delta>0$, we can
  construct in polynomial time a weighted graph $G=(V,E)$, a source vertex
  $s\in V$, and a set $U\subseteq V$ of target vertices such that $\Aut_D$
  is covered at $v$ by $\delta$ iff there is a path in $G$ of weight
  $<\delta$ from $s$ to some vertex in $U$.}

\begin{proof}
Throughout, $\Aut_D$ is $v$-projected, so its first head always reads $v$;
we write each accumulator increment as $a\curr_2+b$ after substituting the
corresponding letter of $v$ for $\curr_1$. Recall that a transition with
$\mov=1$ advances only the first head, so the next transition reads the same
second-head letter again; any RAA may contain such transitions (in the DTW
RAA they align several letters of $v$ with one letter of $w$). The
construction below handles them uniformly.

\emph{Step 1: transition polytopes.}
Fix a transition sequence $\bft=(t_1,\ldots,t_\ell)$ of $\Aut_D$, and let
$n$ be the number of its transitions with $\mov\in\{2,\mathsf{both}\}$,
i.e.\ the length of the second sequence consumed by $\bft$. For each
$i\leq\ell$, let $n_i\in\{1,\ldots,n\}$ be the second-head position that
$t_i$ reads; consecutive transitions linked by $\mov=1$ steps have the same
$n_i$, and hence share one coordinate of $w$. Let $\Pi(\bft)\subseteq\Q^n$
be the set of $w$ such that $\bft$ is an accepting run on $(v,w)$, i.e.\
such that all guards of $\bft$ hold after substituting $v$, substituting
$w_{n_i}$ for $\curr_2$ in $t_i$, and propagating the register updates.
After this substitution every guard atom has the form $w_i\sim w_j$ or
$w_i\sim c$, so $\Pi(\bft)$ is a (not necessarily closed) polytope, and its
closure $\cl(\Pi(\bft))$ in $\bbR^n$ is a closed polytope. The cost of
$\bft$ on $w\in\Pi(\bft)$ is the linear function
$f_{\bft}(w)=\sum_{i=1}^{\ell}(a_iw_{n_i}+b_i)$, where $a_i\curr_2+b_i$ is
the increment of $t_i$. Hence $\Aut_D$ is covered at $v$ by $\delta$ iff
$\inf_{x\in\Pi(\bft)}f_{\bft}(x)<\delta$ for some $\bft$.

\emph{Step 2: optimal points use finitely many constants.}
Let $C$ consist of $0$, the constants in the guards and assignments of
$\Aut_D$, the letters of $v$, and $-\alpha,\alpha$. Since $f_{\bft}$ is
linear and $\cl(\Pi(\bft))$ is bounded (as $\Aut_D$ is $\alpha$-bounded),
its minimum over $\cl(\Pi(\bft))$ is attained at a vertex, i.e.\ a point
that is not a proper convex combination of two other points of
$\cl(\Pi(\bft))$. Let $w$ be a vertex. Call a constraint \emph{tight} at
$w$ if it holds with equality there, and connect two coordinates whenever
a tight constraint relates them. Each connected component contains a
coordinate tied by a tight constraint to a constant of $C$: otherwise,
shifting all coordinates of the component by a small $\pm\epsilon$ keeps
every tight constraint (differences within the component are unchanged) and
every non-tight one (by continuity), exhibiting $w$ as the midpoint of two
feasible points. Hence every coordinate of $w$ lies in $C$, and by
\prettyref{lem:register-provenance} so does every register value along
the run of $\bft$ on $w$.

\emph{Step 3: strict guards.}
A vertex of $\cl(\Pi(\bft))$ need not satisfy the strict and disequality
atoms of $\bft$. We therefore record, besides the values in $C$, an
infinitesimal perturbation: a \emph{tangent tag} is the order of the
coefficients $d_x$ in $x+\epsilon d_x$ among the registers and the pending
second-head letter (defined below) that share a base value in $C$
(constants have coefficient $0$). We keep only tags realisable by
rationals. For fixed $k$ there are $O_k(1)$ tags, and assignments, which
only copy values, update them locally. A guard is taken to hold under a tag
if it holds for all sufficiently small $\epsilon>0$; hence the tags along a
finite path have a common rational realisation of all strict and
disequality atoms.

\emph{Step 4: the graph.}
Let $V=Q\times\{1,\ldots,|v|+1\}\times C^k\times(C\cup\{\star\})\times
\mathcal T_k$. A vertex $(p,i,\mu,z,\tau)$ records the control state, the
first-head position, the register valuation, the \emph{pending} value of the
current second-head letter, and a tangent tag. The pending value is $\star$
when the current letter of $w$ has not been read yet, and otherwise equals
the value it was given; this is how repeated reads of one letter through
$\mov=1$ transitions are forced to agree. For a transition
$t=(p,\varphi,\psi,\chi,q,\mov)$ that needs head~2 (in the sense of
\prettyref{sec:prelim}): if $z=\star$, choose a value $c\in C$ together with
its tangent coefficient; otherwise let $c=z$. A transition that does not need
head~2 reads no letter of $w$ and leaves $z$ unchanged. Add an edge from
$(p,i,\mu,z,\tau)$ to $(q,i',\mu',z',\tau')$ iff $i\leq|v|$ whenever $t$
needs head~1, the tagged valuation satisfies $\varphi(\mu,v_i,c)$, and
$\tau'$ is the tag induced by the update, where
\begin{gather*}
i'=\begin{cases}i+1,&\mov\in\{1,\mathsf{both}\},\\i,&\mov=2,\end{cases}
\qquad
\mu'=\psi(\mu,v_i,c),\\
z'=\begin{cases}\star,&\mov\in\{2,\mathsf{both}\},\\c,&\mov=1\text{ and $t$ needs head~2},\\z,&\text{otherwise}.\end{cases}
\end{gather*}
That is, the pending value is kept across $\mov=1$ transitions and reset
exactly when the second head advances. The edge weight is the increment
$\chi(v_i,c)$ at the base values; it is non-negative because increments are
non-negative on all valuations satisfying the guard, and hence also in the
limit $\epsilon\to0$. The source is $s=(q_0,1,\bar 0,\star,\tau_0)$, where
$\tau_0$ is the unique tag in which all values are constants, and $U$ is the
set of vertices with control state in $F$, first-head position $|v|+1$, and
pending value $\star$ (both heads exhausted).

\emph{Step 5: correctness.}
($\Rightarrow$) Take an accepting run on $(v,w)$ of cost below $\delta$,
with transition sequence $\bft$. By Step~2, $f_{\bft}$ attains its minimum
over $\cl(\Pi(\bft))$ at a vertex $w'$ with coordinates in $C$, and
$f_{\bft}(w')\leq f_{\bft}(w)<\delta$. Since $\Pi(\bft)$ is nonempty, the
direction from $w'$ towards a point of $\Pi(\bft)$ yields a tangent tag
satisfying all strict atoms. Following $\bft$ with the values of $w'$ and
these tags gives an $s$--$U$ path of weight $f_{\bft}(w')<\delta$.
($\Leftarrow$) An $s$--$U$ path of weight below $\delta$ determines a
transition sequence $\bft$ and a point $w'\in\cl(\Pi(\bft))$ with
$f_{\bft}(w')<\delta$; its tags give a rational perturbation into
$\Pi(\bft)$, and by \prettyref{lem:closure} it can be chosen small enough
to keep the cost below $\delta$.

\emph{Step 6: size and witnesses.}
The graph $G$ has $O(|v||Q||C|^{k+1})$ vertices and $O(|v||\Delta||C|^{k+2})$ edges,
as $\mathcal T_k$ has size depending only on $k$. Since all weights are
non-negative, Dijkstra's algorithm decides the instance in polynomial time
for fixed $k$. By Step~5, the weight $d^*$ of a shortest $s$--$U$ path is
the infimum of $\Aut_D(v,w)$ over all $w$. Backtracking that path and
choosing sufficiently close rational representatives yields, for every
$\epsilon>0$, a concrete witness $w$ with $\Aut_D(v,w)<d^*+\epsilon$.
The infimum is attained when no strict or disequality atom is tight along the
path, and then the witness is a closest one; otherwise a minimum-cost witness
need not exist. For example, with a guard $\curr_2<3$ and the cost
$|v_1-w_1|$ for $v_1=5$, the infimum $2$ is approached but never reached.
\end{proof}

\section{Detailed Experimental Results}\label{app:experiments}

%This appendix complements \prettyref{sec:exp}.
\prettyref{app:setup} details the experimental setup.
Appendices~\ref{app:smt-encoding} and~\ref{app:active} describe the SMT-based
and the active learner, and \prettyref{app:learners} compares them with
\learnLS\ for RQ1. Finally, \prettyref{app:rq2-outcomes} lists the
per-configuration outcomes of RQ2.

\subsection{Experimental Setup}\label{app:setup}

All experiments targeted an agreement threshold of $p = 0.95$, an error rate
of $\varepsilon = 0.05$, and an accuracy tolerance of $\gamma = 0.05$ to obtain statistical guarantees,
i.e.\ $n=600$ validation samples and at most $d_{\max}=30$ disagreements
per round (\prettyref{sec:validation-extraction}). The training sample of
\learnLS\ is separate from these validation samples: we draw
$\lceil\ln(2/\varepsilon)/(2\gamma^2)\rceil=738$ sequences, split evenly into
369 positive and 369 negative ones. This size follows the two-sided Hoeffding
bound but is a heuristic choice; the guarantees of
\prettyref{sec:validation-extraction} rely only on the fresh validation
samples. Neighbourhoods are strict,
$B_\delta(v)=\{w:D(v,w)<\delta\}$; we use $\delta=1$ for last-letter,
Manhattan and DTW distance and $\delta=1.5$ for the integer-valued Hamming
and edit distance, so that one substitution, insertion or deletion is
allowed. Each configuration gets a 30-minute learning budget, shared across
refinement rounds, a 60-minute budget for robustness checking, and at most
10 refinement rounds. We report one run with random seed~0, scheduling
six configurations in parallel. The checker compiles each candidate DRA and
metric once per validation phase, explores the product on demand, and forgets
register values that cannot affect a future guard. Concrete witnesses are
checked against the original automata before network confirmation.

We trained the LSTM and transformer models using binary cross-entropy loss
and the \textsc{AdamW} optimiser with a learning rate of~0.001 and a dropout rate
of~0.1. Each model was trained for up to 50{,}000 steps or terminated early once
the validation set achieved 100\% accuracy for three consecutive epochs.
The LSTM models used 2--4 recurrent layers with hidden sizes ranging from 64 to
256, while the transformer encoders employed 2--4 attention heads, 2--6 layers,
and model dimensions between 64 and 128.
All models were trained on four NVIDIA Quadro RTX~8000 GPUs, except for
eight networks (\nnTransformer\ on $L_4$, $L_5$, $L_6$, $S_4$, $S_9$,
$S_{11}$ and \nnLSTM\ on $S_4$, $S_{11}$), which were trained on a CPU with
the same optimiser on 40{,}000 balanced sequences drawn from the benchmark
generators, and the two \textsf{MNIST} networks, which were trained on a CPU
on the 13{,}484 balanced row-mean sequences of the \textsf{MNIST} training
images and read the raw intensities. For \textsf{MNIST}, the learner works
over the rationals and may use the constant~8 in guards. For DRA learning, we use $T_0=1$, $\rho=0.995$,
$T_{\min}=10^{-5}$, $\textit{max\_iter}=5\cdot10^4$.

\subsection{SMT-Based Learner}\label{app:smt-encoding}

Following SAT/SMT-based DFA synthesis~\cite{heule2010exact,neider2013regular},
the SMT-based learner encodes the existence of a $k$-DRA with $n$ states
$Q=\{q_0,\ldots,q_{n-1}\}$ over a constant set $C$ that is consistent with a
sample set $S=(S_+,S_-)$ as a formula $\Phi_{n,k}(S)$, and reads a DRA off a
model.

\paragraph{Guards and updates.}
For a fixed $k$ and $C$, only finitely many guards and updates are needed. Over
the rational languages $S_i$ and \textsf{MNIST}, a guard is an interval
$\lambda\wedge\rho$ whose lower bound $\lambda$ is $\curr>v$, $\curr\geq v$, or
absent, and whose upper bound $\rho$ is $\curr<v$, $\curr\leq v$, or absent,
with $v\in\bar r\cup C$; over the integer languages $L_i$, a guard is
$\curr=v$, $\curr\neq v$, $\top$, or $\bot$. Let $\Gamma$ be this finite set of
guards; an empty guard, such as $\curr>v\wedge\curr<v$, denotes a missing
transition. An update in the finite set $\Upsilon$ assigns one register a
register value or $\curr$ and leaves the others unchanged.

\paragraph{Variables.}
For every pair of states $p,q$, an enumerated variable $\gamma_{p,q}$ over
$\Gamma$ and an enumerated variable $\upsilon_{p,q}$ over $\Upsilon$ select the
guard and update of the transition from $p$ to $q$, and a Boolean $f_q$ states
that $q$ is accepting. For every prefix $u\in\mathrm{Pref}(S)$ of a sample, a
Boolean $x_{u,q}$ states that the run on $u$ reaches $q$, and rational
variables $\bar r_u=(r_{u,1},\ldots,r_{u,k})$ hold the register values after
$u$.

\paragraph{Constraints.}
$\Phi_{n,k}(S)$ is the conjunction of the following formulas, where
$g,g'\in\Gamma$, $\upsilon\in\Upsilon$, $ua\in\mathrm{Pref}(S)$ with
$a\in\Q$, and $p,q,q'\in Q$ range over all values:
\begin{align}
	&\textstyle|\{q : \gamma_{p,q}\text{ non-empty}\}|\leq m \text{ for every } p
	\label{smt:limit}\\
	&\gamma_{p,q}=g\wedge\gamma_{p,q'}=g' \rightarrow \bot
	\quad\text{for } q\neq q' \text{ and overlapping } g,g'
	\label{smt:det}\\
	&x_{\varepsilon,q_0}\wedge\textstyle\bigwedge_{i}r_{\varepsilon,i}=0
	\label{smt:init}\\
	&\textstyle\bigvee_{q}x_{u,q}\quad\text{for every } u\in\mathrm{Pref}(S_+)
	\label{smt:pos}\\
	&x_{u,p}\wedge\gamma_{p,q}=g\wedge\upsilon_{p,q}=\upsilon\wedge g(\bar r_u,a)
	\rightarrow x_{ua,q}\wedge\upsilon(\bar r_u,a,\bar r_{ua})
	\label{smt:step}\\
	&x_{u,p}\wedge\gamma_{p,q}=g\wedge\neg g(\bar r_u,a)\rightarrow\neg x_{ua,q}
	\label{smt:block}\\
	&\textstyle\bigwedge_{w\in S_+}(x_{w,q}\rightarrow f_q)\wedge
	\bigwedge_{w\in S_-}(x_{w,q}\rightarrow\neg f_q)
	\label{smt:acc}
\end{align}
Formula~\eqref{smt:limit} bounds the number of outgoing transitions by the
benchmark's edge limit $m$ (for some benchmarks, with equality). Formula~\eqref{smt:det}
makes the automaton deterministic: two guards \emph{overlap} if some valuation
of $\bar r$ and $\curr$ satisfies both, which is decided once per pair of guards.
Formula~\eqref{smt:init} fixes the initial configuration, and~\eqref{smt:pos}
requires every prefix of a positive sample to have a run. Formulas~\eqref{smt:step}
and~\eqref{smt:block} encode one step of a run: an enabled transition must be
taken and applies its update, written $\upsilon(\bar r_u,a,\bar r_{ua})$, while a
disabled one cannot be taken. Formula~\eqref{smt:acc} fixes the labels.
Negative samples need not have a run, so no explicit rejecting sink state
has to be synthesised.

\begin{lemma}\label{lem:smt-correct}
	If $\mathcal M\models\Phi_{n,k}(S)$, then the automaton $\Aut_{\mathcal M}$
	with states $Q$, initial state $q_0$, a transition
	$(p,\mathcal M(\gamma_{p,q}),\mathcal M(\upsilon_{p,q}),q)$ for every
	non-empty $\mathcal M(\gamma_{p,q})$, and final states
	$\{q:\mathcal M(f_q)\}$ is an $n$-state $k$-DRA consistent with $S$.
\end{lemma}
\begin{proof}
$\Aut_{\mathcal M}$ is deterministic by~\eqref{smt:det}. We show by induction on
$u\in\mathrm{Pref}(S)$: if the run of $\Aut_{\mathcal M}$ on $u$ ends in state
$q$ with register values $\mu$, then $\mathcal M(x_{u,q})$ holds and
$\mathcal M(\bar r_u)=\mu$. The base case is~\eqref{smt:init}. For the step,
let the run on $ua$ take the transition from $p$ to $q$ with guard $g$ and
update $\upsilon$. By the induction hypothesis, $\mathcal M(x_{u,p})$ holds and
$g(\mathcal M(\bar r_u),a)$ is true, so~\eqref{smt:step} gives
$\mathcal M(x_{ua,q})$ and $\mathcal M(\bar r_{ua})=\upsilon(\mu,a)$.

Now let $w\in S_+$. Its run exists: otherwise, for the longest prefix $u$
with a run, ending in $p$, and the next letter $a$, every guard from $p$ is
false at $a$, so~\eqref{smt:block} gives $\neg\mathcal M(x_{ua,q})$ for all
$q$, contradicting~\eqref{smt:pos}. By the claim, the run ends in some $q$ with
$\mathcal M(x_{w,q})$, so $q$ is accepting by~\eqref{smt:acc}. Finally, let
$w\in S_-$. If $w$ has no run, it is rejected; otherwise its run ends in some $q$
with $\mathcal M(x_{w,q})$, which is rejecting by~\eqref{smt:acc}.
\end{proof}

\paragraph{Incremental solving.}
Since $\Phi_{n,k}(S)$ grows with the number of prefixes, the learner adds
samples in batches of 15 positive and 15 negative ones ($L_i$) or 5 and 5
($S_i$, \textsf{MNIST}) and calls Z3 after each batch. After each satisfiable
batch, it decodes the model as in \prettyref{lem:smt-correct} and stops once
the DRA's accuracy on all training samples exceeds 95\% ($L_i$) or 90\%
($S_i$, \textsf{MNIST}), or after seven batches. This early stop keeps the
formulas small, and explains why some of its DRAs are less accurate than those
of \learnLS.

\subsection{Active Learner}\label{app:active}

RALib~\cite{dierl2024scalable} implements the tree-based algorithm
SL$^\lambda$, a successor of SL$^*$~\cite{10.1007/978-3-319-10431-7_18}, and
acts as an $L^*$-style learner for register automata. It learns from
\emph{membership queries}, answered here by the ground-truth language, and
adds states and registers only when counterexamples require them, so it
chooses their numbers itself. Given the registers learned so far, it enumerates the
order types of $\bar r$ and $\curr$ (we use its equality theory for $L_i$ and
its inequality theory for $S_i$) and realises each by a concrete rational, for
instance a midpoint between known values, to turn symbolic queries into
membership queries. Equivalence queries are replaced by sampling, similarly
to \prettyref{sec:validation-extraction} but with a two-sided Hoeffding
bound: each hypothesis is compared with the language on
$n=\lceil\ln(2/\varepsilon)/(2\gamma^2)\rceil=738$ fresh samples
($\varepsilon=\gamma=0.05$, length at most 10), every disagreement is returned
as a counterexample, and learning stops once the agreement reaches
$1-\varepsilon$. Its DRAs include an explicit rejecting sink state. RALib cannot
be applied to \textsf{MNIST}, for which no membership oracle exists.

\subsection{Comparison of DRA Learners}\label{app:learners}

\prettyref{tab:learners} compares the local-search learner \learnLS\ with the
SMT-based learner (\prettyref{app:smt-encoding}) and the active learner
(\prettyref{app:active}) on the RQ1 benchmarks.

\paragraph{Protocol.}
\learnLS\ and the SMT-based learner receive the same 369 positive and 369
negative samples (length at most 10, seed~0) and the numbers of states and
registers of \prettyref{tab:learn-res}. All three learners have a 30-minute
budget, and all DRAs are scored with parallel register updates on the same
balanced held-out samples as \prettyref{tab:learn-res}.

\paragraph{Results.}
RALib is exact on $L_1$--$L_7$ and $S_1$--$S_8$ within 30~s, but times out on
$S_9$ and $S_{10}$ and stops on $S_{11}$ with an internal error: its
inequality theory rejects a compound guard. The SMT-based learner reaches
91.05--100\% on $L_1$--$L_7$ and $S_1$--$S_8$ but is slower and times out on $S_9$--$S_{11}$.
\learnLS\ is the only learner that handles all benchmarks; it is the least
accurate on $L_5$ (79.0\%), where its search ends in a local optimum.

\begin{table}[htbp]
	\caption{Learner comparison. \textit{Acc}: balanced test accuracy (\%);
	\textit{Time}: learning time (s); \textit{S/R}: states/registers.}
	\label{tab:learners}
	\centering\scriptsize
	\setlength{\tabcolsep}{3pt}\renewcommand{\arraystretch}{0.85}
	\begin{tabular}{l rrc|rrc|rrc}
		\toprule
		& \multicolumn{3}{c|}{\learnLS} & \multicolumn{3}{c|}{SMT-based}
		& \multicolumn{3}{c}{RALib (active)}\\
		Lang & Acc & Time & S/R & Acc & Time & S/R & Acc & Time & S/R\\
		\midrule
		$L_{1}$ & 100.00 & 1.0 & 2/1 & 99.80 & 23.7 & 2/1 & 100.00 & 4.7 & 3/1 \\
		$L_{2}$ & 99.95 & 20.1 & 4/1 & 98.40 & 78.5 & 4/1 & 100.00 & 5.2 & 5/2 \\
		$L_{3}$ & 100.00 & 0.5 & 4/1 & 99.75 & 55.1 & 4/1 & 100.00 & 6.5 & 6/2 \\
		$L_{4}$ & 99.45 & 27.9 & 4/2 & 99.85 & 507.2 & 4/2 & 100.00 & 3.5 & 4/1 \\
		$L_{5}$ & 79.00 & 753.3 & 4/1 & 100.00 & 82.3 & 4/1 & 100.00 & 4.0 & 4/2 \\
		$L_{6}$ & 100.00 & 11.7 & 3/2 & 96.65 & 328.7 & 3/2 & 100.00 & 2.9 & 3/1 \\
		$L_{7}$ & 100.00 & 335.8 & 5/1 & 99.80 & 544.4 & 5/1 & 100.00 & 8.0 & 6/2 \\
		$S_{1}$ & 100.00 & 2.1 & 2/1 & 96.60 & 15.2 & 2/1 & 100.00 & 3.7 & 3/1 \\
		$S_{2}$ & 100.00 & 39.9 & 2/1 & 99.45 & 30.5 & 2/1 & 100.00 & 4.2 & 3/1 \\
		$S_{3}$ & 100.00 & 0.3 & 2/1 & 100.00 & 48.7 & 2/1 & 100.00 & 8.1 & 3/1 \\
		$S_{4}$ & 100.00 & 2.7 & 2/1 & 91.05 & 45.0 & 2/1 & 100.00 & 5.4 & 3/1 \\
		$S_{5}$ & 100.00 & 353.9 & 4/1 & 95.60 & 401.0 & 4/1 & 100.00 & 5.4 & 5/1 \\
		$S_{6}$ & 100.00 & 310.5 & 4/1 & 94.05 & 335.5 & 4/1 & 100.00 & 5.2 & 5/1 \\
		$S_{7}$ & 99.80 & 263.0 & 5/1 & 95.25 & 453.2 & 5/1 & 100.00 & 9.8 & 6/1 \\
		$S_{8}$ & 100.00 & 83.0 & 6/1 & 94.00 & 651.7 & 6/1 & 100.00 & 29.1 & 7/1 \\
		$S_{9}$ & 99.90 & 543.5 & 5/2 & \multicolumn{3}{c|}{timeout} & \multicolumn{3}{c}{timeout} \\
		$S_{10}$ & 99.60 & 504.3 & 5/2 & \multicolumn{3}{c|}{timeout} & \multicolumn{3}{c}{timeout} \\
		$S_{11}$ & 99.40 & 387.4 & 5/2 & \multicolumn{3}{c|}{timeout} & \multicolumn{3}{c}{crash} \\
		\textsf{MNIST} & 95.37 & 832.1 & 4/1 & 94.13 & 1582.9 & 4/1 & \multicolumn{3}{c}{n/a} \\
		\bottomrule
	\end{tabular}
\end{table}

\subsection{RQ2 Outcomes}\label{app:rq2-outcomes}

\prettyref{tab:rq2} lists, for each metric, the configurations of the
end-to-end procedure that did not end non-robust; the per-metric totals are
also given in \prettyref{tab:loop}.

\begin{table}[htbp]
	\caption{RQ2 outcomes per metric. Every configuration not listed is
	non-robust, with a witness confirmed on the network. (L)/(T): for \nnLSTM\
	or \nnTransformer\ only; otherwise for both.}
	\label{tab:rq2}
	\centering\scriptsize
	\renewcommand{\arraystretch}{1.1}\setlength{\tabcolsep}{5pt}
	\begin{tabular}{l c >{\raggedright\arraybackslash}p{5.2cm}
		>{\raggedright\arraybackslash}p{2.7cm}}
		\toprule
		Metric & $+/-/?$ & Robust (PAC-accepted surrogate) & Inconclusive\\
		\midrule
		Last letter & 21/14/3 & $L_2$\,(L), $L_3$, $S_1$, $S_2$\,(T), $S_5$\,(L),
		$S_6$--$S_{11}$, \textsf{MNIST} & $L_7$, $S_5$\,(T)\\
		Edit & 0/36/2 & -- & $L_5$\,(T), \textsf{MNIST}\,(T)\\
		Hamming & 2/35/1 & $L_3$ & $L_6$\,(T)\\
		Manhattan & 15/19/4 & $L_3$, $S_1$\,(L), $S_5$\,(L), $S_6$, $S_7$\,(T),
		$S_8$--$S_{11}$ & $S_5$\,(T), $S_7$\,(L), \textsf{MNIST}\\
		DTW & 0/38/0 & -- & --\\
		\bottomrule
	\end{tabular}
\end{table}

\fi
\end{document}